\documentclass[letterpaper]{article} 
\usepackage[preprint]{aaai2027}  
\usepackage[hyphens]{url}  
\usepackage{graphicx} 
\usepackage{natbib}  
\usepackage{caption} 
\usepackage{algorithm}
\usepackage{algorithmic}

\usepackage{newfloat}
\usepackage{listings}
\DeclareCaptionStyle{ruled}{labelfont=normalfont,labelsep=colon,strut=off} 
\floatstyle{ruled}
\newfloat{listing}{tb}{lst}{}
\floatname{listing}{Listing}

\usepackage{booktabs}
\usepackage{amsmath,amssymb,amsthm}
\usepackage[table]{xcolor}
\usepackage{multirow}

\newtheorem{theorem}{Theorem}
\newtheorem{proposition}{Proposition}
\newtheorem{corollary}{Corollary}

\title{
PIT-SUN: A Deployable Empirical Marginal Transform Framework with Expectation-Consistent Recovery for Regression in Recommender Systems}
\author{
    Mingyu Zhao\textsuperscript{\rm 1}\thanks{Work performed during an internship at Kuaishou Technology.},
    Zhaohan Li\textsuperscript{\rm 2},
    Zhenxiong Miao\textsuperscript{\rm 2},
    Xu Zhang\textsuperscript{\rm 2},
    Dewei Leng\textsuperscript{\rm 2},
    Yanan Niu\textsuperscript{\rm 2},
    Kun Gai\textsuperscript{\rm 2}
}
\affiliations{
    \textsuperscript{\rm 1}Renmin University of China\\
    \textsuperscript{\rm 2}Kuaishou Technology
}

\begin{document}

\maketitle
\begin{abstract}
Estimating original-space conditional expectations is central to value-driven recommender systems, including dwell time, GMV, and LTV forecasting. Standard MSE is expectation-consistent in principle, but its gradients become unstable on heavy-tailed, zero-inflated, and multimodal targets, causing mean collapse and tail shrinkage. Target transformation alleviates this scale conflict, yet any useful nonlinear marginal transform loses expectation consistency under direct inversion. This is not an implementation oversight: a direct inverse-transform estimator is universally expectation-consistent only when the inverse transform is affine, which cannot simultaneously provide bounded tail compression. Existing conditionally linear recovery methods restore expectation consistency, but still leave open which coordinate, inverse lookup, recovery base, and deployment monitor should be selected for sparse complex marginals. We propose \textbf{P}robability-\textbf{I}ntegral-\textbf{TranS}formed \textbf{Un}biased recovery (\textbf{PIT-SUN}), a deployable empirical marginal recovery framework. PIT-SUN uses one empirical marginal table to define a bounded normal-score coordinate, its inverse-quantile lookup, a variance-controlled recovery base, and drift monitoring, then applies multiplicative SUN recovery to estimate the original-space expectation instead of directly inverting transformed predictions. Experiments on synthetic distributions, public benchmarks, large-scale industrial datasets, and online deployment show robust improvements in point accuracy, calibration, and ranking quality with lightweight deployment overhead.
\end{abstract}

\section{Introduction}

Many recommender tasks---including dwell time, GMV, and LTV forecasting---predict continuous business values for ranking, allocation, and multi-objective decisions. Here, \textbf{dwell time} denotes user video watching time~\cite{Davidson2010YouTube,Yi2014DwellTime,Wu2018BeyondViews,Wang2020DwellTime}, while LTV and GMV support customer acquisition and value-driven marketing~\cite{ziln,Mdme,OptDist2024}. As shown in Figure~\ref{fig:overview}, these targets are often sparse, heavy-tailed, multimodal, and context-dependent. Although original-space MSE has $m(x):=\mathbb{E}[Y\mid X=x]$ as its Bayes optimum, finite-sample gradients are exposed to extreme values, causing \textit{mean collapse} or \textit{tail shrinkage}.

Monotonic transformations such as log, sqrt, or Box-Cox~\cite{Bartlett1936Sqrt,Bartlett1947Transform,Sakia1992BoxCox} stabilize labels but introduce retransformation bias~\cite{NeymanScott1960Bias,Duan1983Smearing}: $T^{-1}(\mathbb{E}[T(Y)\mid x]) \neq m(x)$. This is structural: if direct inversion were expectation-preserving for all conditional distributions, $T^{-1}$ would have to be affine, ruling out nonlinear tail compression. Thus recommender regression needs more than a better scalar transform. Existing recovery methods such as TranSUN/GTS~\cite{Transun/gts} restore expectation consistency, but still rely on manual choices of transformed target and recovery base. The gap is a deployable \emph{closure}: the coordinate, inverse lookup, recovery base, and empirical monitor must be specified jointly.

We formalize a \textbf{deployable empirical marginal recovery closure}: one empirical marginal table defines the stable coordinate, inverse lookup, recovery base, and drift monitor while preserving $m(x)=\mathbb{E}[Y\mid x]$. This explains why standard transformations, PIT-only regression, and manual-transform recovery each remain incomplete. We instantiate this closure with \textbf{P}robability-\textbf{I}ntegral-\textbf{TranS}formed \textbf{Un}biased recovery (\textbf{PIT-SUN}), summarized in Figure~\ref{fig:overview}. Our main contributions are:

\begin{figure*}[t]
    \centering
    \includegraphics[width=\textwidth]{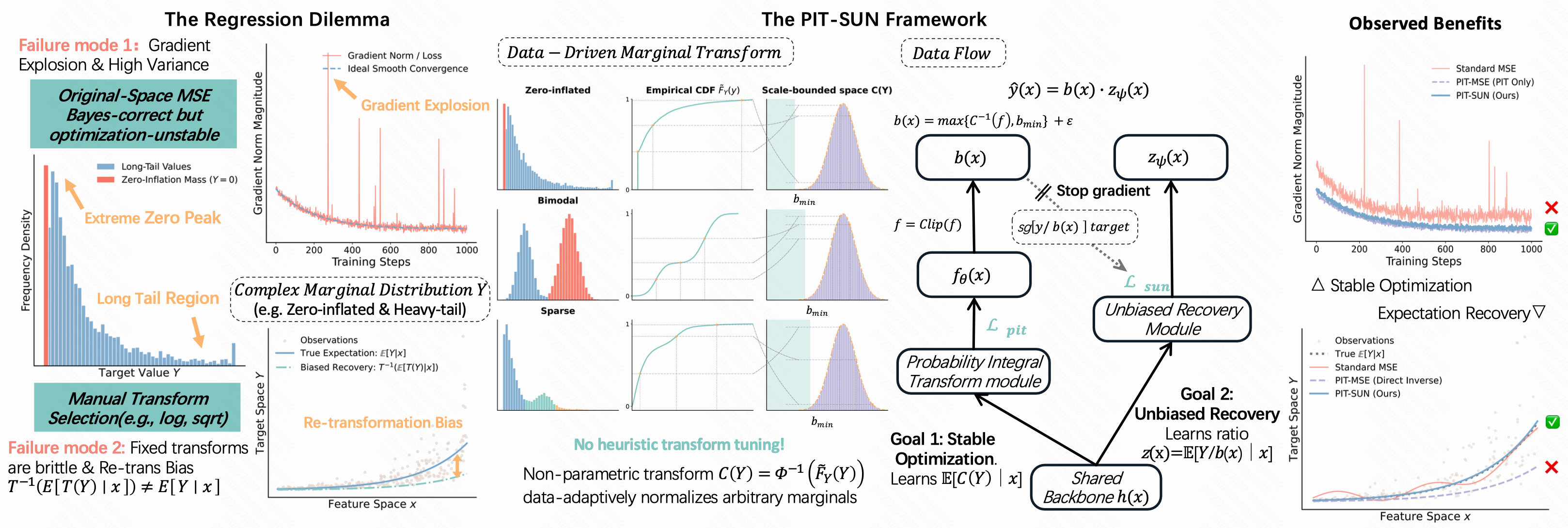}
    \caption{\textbf{Overview of PIT-SUN.} Empirical PIT provides a bounded coordinate, and SUN recovery maps it back to the original-space expectation through a stabilized inverse-quantile base.}
    \label{fig:overview}
\end{figure*}

\begin{itemize}
    \item \textbf{Impossibility-guided formulation.} We show that nonlinear stable coordinates cannot be universally expectation-preserving under direct inversion, motivating a coordinate--base--recovery closure rather than another scalar transform.
    \item \textbf{PIT-SUN closure.} We use one CDF table for the PIT coordinate, inverse lookup, inverse-quantile recovery base, and drift monitoring, with lower-bound stabilization and stop-gradient isolation.
    \item \textbf{Empirical validation.} We validate PIT-SUN on synthetic, public, industrial offline, and online settings, with PIT-TranSUN and component diagnostics isolating why the complete closure is needed.
\end{itemize}

\section{Related Work}

\subsubsection{Regression Targets and Existing Adaptation Strategies}
Continuous recommender targets such as dwell time, GMV, and LTV are sparse, heavy-tailed, multimodal, and context-dependent. Existing adaptations fall into three groups.

\textbf{Distributional and value-aware modeling.} ZILN~\cite{ziln} models customer value with a zero-inflated lognormal; WLR~\cite{wlr} weights impressions by dwell time; EGMN~\cite{Egmn} uses an Exponential-Gaussian mixture for short-video dwell time; MDME~\cite{Mdme} decomposes LTV into sub-distribution experts; and OptDist~\cite{OptDist2024} adaptively selects sub-distributions. Mixture density networks and normalizing flows~\cite{Bishop1994MDN,Papamakarios2021NF} further expand conditional output families. These methods are expressive, but rely on distributional assumptions, mixture design, or selection mechanisms rather than a lightweight point-expectation recovery path.

\textbf{Discretized and structured regression.} TPM~\cite{Tpm} decomposes dwell time into tree-structured classifications; CREAD~\cite{Cread} combines adaptive discretization with restoration; GR~\cite{Gr} predicts time tokens; and CCOR-Net~\cite{ccor-net} decouples ranking and residual regression for LTV. These methods reduce local fitting difficulty, but introduce buckets, routing, vocabularies, or restoration modules.

\textbf{Debiasing and quantile preference signals.} D2Q~\cite{D2Q2022}, DVR~\cite{DVR2022}, CVRDD~\cite{CVRDD2023}, D2Co~\cite{D2Co2023}, CWM~\cite{CWM2024}, and RAD~\cite{RAD} correct dwell-time bias, noisy watching, truncation, or relative-advantage labels. This line is complementary to PIT-SUN, which targets expectation-consistent regression under complex empirical marginals.

\subsubsection{Target Transformation and Expectation Recovery}
Another route keeps point regression but transforms labels. Classical square-root, logarithmic, and Box-Cox transformations~\cite{Bartlett1936Sqrt,Bartlett1947Transform,Sakia1992BoxCox} stabilize non-Gaussian responses; delta-method and smearing estimators~\cite{NeymanScott1960Bias,Duan1983Smearing} correct retransformation bias. Loss-level reweighting is complementary, but may trade aggregate calibration for tail emphasis. The core issue remains that direct inversion estimates $T^{-1}(\mathbb{E}[T(Y)\mid x])$ instead of $\mathbb{E}[Y\mid x]$.

TranSUN/GTS~\cite{Transun/gts} restore expectation consistency through conditionally linear recovery, while conditional transformation models~\cite{Hothorn2014CTM,Most2016CLTM} study related transformation ideas for distributional or interval estimation. Yet existing recovery-based point-regression methods typically depend on manually selected transformations and recovery bases, treating the coordinate and recovery path as decoupled analytical components. Simply replacing the manual transform with an empirical PIT map is not sufficient, as shown by the PIT-TranSUN diagnostic baseline and the drift analyses in Appendix~D.5. PIT-SUN instead treats the empirical marginal table as a unified recovery closure, jointly defining the stable coordinate, inverse lookup, stabilized recovery base, and deployment monitor, rather than serving as another manual transform or empirical plug-in.
\begin{table}[t]
\centering
{\fontsize{9pt}{10pt}\selectfont\setlength{\tabcolsep}{2pt}
\begin{tabular}{llll}
\toprule
Method & Coord. source & Recovery path & Deploy. object \\
\midrule
T-MSE & manual $T$ & none & transform rule \\
TranSUN & manual $T$ & cond.-linear & $T$ + base \\
PIT-MSE & empirical CDF & direct inverse & CDF table \\
\textbf{PIT-SUN} & empirical CDF & SUN + inv.-quantile & shared CDF table \\
\bottomrule
\end{tabular}
}
\caption{Positioning of PIT-SUN among transform-based recovery methods.}
\label{tab:method_positioning}
\end{table}

\subsubsection{Probability Integral Transform}
The PIT, normal-score transform, and rank Gaussianization map arbitrary marginals to uniform or normal references. In supervised regression, however, PIT-only prediction remains biased after direct inversion: $C^{-1}(\mathbb{E}[C(Y)\mid x]) \neq \mathbb{E}[Y\mid x]$. PIT-SUN therefore uses PIT not as a standalone transform, but as the coordinate in an expectation-recovery closure. Unlike quantile-regression and quantile-debiasing methods~\cite{Koenker1978QR,Meinshausen2006QRF,CQE2025,AlignPxtr2025,RAD}, which estimate conditional quantiles or debiased preference labels, PIT-SUN uses the empirical marginal CDF to support stable regression, recovery-base construction, and original-space expectation recovery.

\section{Preliminaries}
\label{sec:preliminaries}
\subsection{Regression Dilemma: Stability vs. Expectation}
Let $X \in \mathcal{X}$ be the feature vector and $Y \in \mathbb{R}_{\ge 0}$ be the non-negative target. We denote the original-space conditional expectation by $m(x):=\mathbb{E}[Y\mid X=x]$. Original-space MSE is statistically aligned with this target, but in recommender data the labels are often heavy-tailed, sparse, or multimodal. The resulting gradients are dominated by rare large values or zero plateaus, making the correct objective difficult to optimize reliably in finite samples.

A common response is to learn in a transformed space. For a monotone transform $T$, transformed MSE minimizes
\begin{equation}
\mathcal{L}_{T\text{-MSE}}
=\mathbb{E}_{(x,y)}\bigl[(f(x)-T(y))^2\bigr],
\end{equation}
whose population optimizer is
\begin{equation}
f_T^*(x)=\mathbb{E}[T(Y)\mid x].
\end{equation}
This can produce a much better-conditioned learning coordinate, but direct inversion gives $T^{-1}(f_T^*(x))$ rather than $m(x)$ in general. The issue is not a particular choice of log, sqrt, or PIT; it is that nonlinear compression changes the expectation being estimated. Thus the basic dilemma is: the original scale has the right target but unstable optimization, while the transformed scale has a stable coordinate but biased direct recovery.

\subsection{Conditionally Linear Recovery}
Recovery-based methods resolve the expectation side by making the last step conditionally linear in $Y$. Given a learned representation $f_T(x)$ and a positive scale function $b(x)$ fixed conditional on $x$, the recovery branch learns the proportional target
$
z(x) \approx \mathbb{E}\left[\frac{Y}{b(x)}\middle|x\right].$
The final prediction is then
\begin{equation}
\hat{y}(x)=z(x)b(x),
\end{equation}
so at the population optimum $\hat{y}(x)=m(x)$. This principle explains how expectation consistency can be restored without giving up a nonlinear learning coordinate.

However, the principle does not specify which coordinate to learn or which base $b(x)$ to use. These choices determine whether recovery is numerically stable under sparse and heavy-tailed marginals. PIT-SUN therefore treats the coordinate and the recovery base as a coupled design problem: the next section derives both from one empirical marginal table and adds the stabilization needed for deployment.

\section{Methodology}
\label{sec:methodology}

Figure~\ref{fig:overview} gives an overview of PIT-SUN; the following subsections define its empirical PIT coordinate, inverse-quantile base, and SUN recovery objective.

\subsection{Empirical PIT: Scale-Bounded Learning Space}
The empirical PIT branch replaces manual analytical transforms with a data-induced coordinate. Let $\hat{F}_Y$ be the empirical CDF, $\Phi$ the standard normal CDF, and $\delta$ a boundary clipping probability. We define
\begin{equation}
\tilde{F}_Y(y)=\operatorname{clip}\bigl(\hat{F}_Y(y),\delta,1-\delta\bigr),
\quad
C(y)=\Phi^{-1}(\tilde{F}_Y(y)).
\end{equation}
Its generalized inverse is implemented by the empirical quantile table,
\begin{equation}
C^{-1}(w)=\hat{F}_Y^{-1}\bigl(\Phi(w)\bigr).
\end{equation}
The branch $f_\theta(x)$ is trained with
\begin{equation}
\mathcal{L}_{\mathrm{PIT}}
=\mathbb{E}_{(x,y)}\bigl[(f_\theta(x)-C(y))^2\bigr].
\end{equation}
This module directly targets the first regression difficulty: regardless of the native target scale, clipped PIT satisfies $|C(Y)|\le a_\delta=\Phi^{-1}(1-\delta)$. It therefore turns heavy-tailed amplitudes into bounded normal-score coordinates while preserving marginal order.

\subsection{Inverse-Quantile Base and SUN Recovery}
The PIT branch alone is not used as the final predictor, because $C^{-1}(\mathbb{E}[C(Y)\mid x])$ is still a biased estimator of $m(x)$. Instead, the predicted PIT coordinate is converted into an original-scale recovery base. With $\bar{f}(x)=\operatorname{clip}(f_\theta(x),-a_\delta,a_\delta)$, we define
\begin{equation}
\label{eq:base}
b(x)=\max\Bigl(C^{-1}(\bar{f}(x)),\,b_{\min}\Bigr)+\varepsilon .
\end{equation}
This inverse-quantile base closes the gap between the stable PIT coordinate and the original business scale. It is a conditioning device, not another transform family: it maps the learned marginal rank into an original-scale denominator, while $b_{\min}$ prevents zero-plateau denominator collapse and controls recovery-label variance.

The SUN recovery branch learns the proportional target under gradient isolation:
\begin{equation}
\mathcal{L}_{\mathrm{SUN}}
=\mathbb{E}_{(x,y)}\left[
\left(z_\psi(x)-\frac{y}{\operatorname{sg}[b(x)]}\right)^2
\right],
\end{equation}
where $\operatorname{sg}[\cdot]$ denotes stop-gradient. In implementation we stop the denominator path: the proportional label is formed as $r=y/\operatorname{sg}[b(x)]$. Equivalently, the recovery loss does not backpropagate through the inverse-quantile lookup, clipping operation, or $b_{\min}$ floor, while the ratio head $z_\psi$ remains trainable. This branch addresses the second regression difficulty: it restores a conditionally linear path to $m(x)$ while preventing high-variance proportional labels from corrupting the PIT geometry. Without this isolation, high-variance ratio residuals can distort the PIT coordinate through the denominator path, increasing the joint-training gap analyzed after Theorem~3.

\subsection{PIT-SUN Training and Inference}
The full objective is
\begin{equation}
\mathcal{L}_{\mathrm{PIT\text{-}SUN}}
=\mathcal{L}_{\mathrm{PIT}}+\lambda\mathcal{L}_{\mathrm{SUN}}.
\end{equation}
At inference time, PIT-SUN multiplies the learned recovery ratio and the inverse-quantile base:
\begin{equation}
\hat{y}(x)=z_\psi(x)b(x).
\end{equation}
For non-negative targets, the ratio head is implemented with a non-negative output link, e.g., softplus or clipped ReLU, so that $\hat{y}(x)\ge0$. The resulting estimator keeps the optimization benefit of PIT and the expectation target of original-space MSE. Its deployment overhead is limited to a sliding-window quantile table and two lightweight regression heads. The full procedure is summarized in Algorithm~\ref{alg:pitsun}.

\begin{algorithm}[tb]
\caption{Training and Inference of PIT-SUN}
\label{alg:pitsun}
\textbf{Input}: Training data $\{(x_i, y_i)\}_{i=1}^N$, test data $X_{\text{test}}$, clipping probability $\delta$, quantile $q_b$, weight $\lambda$.
\begin{algorithmic}[1]
\STATE Construct empirical CDF $\hat{F}_Y$ from $\{y_i\}$.
\STATE Define $C(y) = \Phi^{-1}(\operatorname{clip}(\hat{F}_Y(y), \delta, 1-\delta))$.
\STATE Construct lookup table $C^{-1}(w) = \hat{F}_Y^{-1}(\Phi(w))$.
\STATE Set $b_{\min}$ as the $q_b$-quantile of the positive targets.
\FOR{each minibatch $(x, y)$}
    \STATE $c \leftarrow C(y)$
    \STATE $f \leftarrow f_\theta(x)$, \quad $\bar{f} \leftarrow \operatorname{clip}(f, -a_\delta, a_\delta)$
    \STATE $b \leftarrow \max(C^{-1}(\bar{f}), b_{\min}) + \varepsilon$ \quad \COMMENT{floor convention used throughout}
    \STATE $r \leftarrow y / \operatorname{sg}[b]$ \quad \COMMENT{stop only the base/lookup path}
    \STATE $\mathcal{L}_{\text{PIT}} \leftarrow \|f - c\|_2^2$, \quad $\mathcal{L}_{\text{SUN}} \leftarrow \|z_\psi(x) - r\|_2^2$
    \STATE Update $\theta, \psi$ by minimizing $\mathcal{L}_{\text{PIT}} + \lambda \mathcal{L}_{\text{SUN}}$
\ENDFOR
\STATE \textbf{Inference for $x \in X_{\text{test}}$:}
\STATE $f \leftarrow f_\theta(x)$, \quad $\bar{f} \leftarrow \operatorname{clip}(f, -a_\delta, a_\delta)$
\STATE $b \leftarrow \max(C^{-1}(\bar{f}), b_{\min}) + \varepsilon$
\STATE \textbf{return} $\hat{y} = z_\psi(x) \cdot b$
\end{algorithmic}
\end{algorithm}

\subsection{Practical Handling of Sparse and Tied Targets}
Repeated zeros and ties are common in GMV targets. PIT-SUN handles atoms with a generalized empirical PIT,
\begin{equation}
\tilde{F}_Y(y)=\hat{F}_Y(y^-)+\rho\bigl(\hat{F}_Y(y)-\hat{F}_Y(y^-)\bigr),\rho\in[0,1]
\end{equation}
using deterministic mid-ranks by default and randomized ranks only as an optional tie-breaking variant. This keeps sparse targets within quantile intervals without adding a hand-crafted smoothing transform. For extreme zero-inflation, we use PIT-SUN-ZI, which factorizes occurrence and positive amount: a hurdle branch estimates $P(Y>0 \mid x)$, while PIT-SUN on positive samples estimates $\mathbb{E}[Y \mid x, Y>0]$. The final prediction is $\hat y(x)=P(Y>0 \mid x)\mathbb{E}[Y \mid x, Y>0]$; implementation details are discussed in Appendix~B.7-8.

\section{Theoretical Justification}
\label{sec:theory}
We summarize the formal properties governing the PIT-SUN closure. Complete proofs and analyses of empirical boundaries (e.g., clipping, tied atoms) are deferred to Appendices A and B.

\begin{theorem}[No Nonlinear Direct Inversion]
Let $T$ be strictly monotone on an interval $\mathcal{I}$ containing the support of $Y\mid x$. Let $g=T^{-1}$ be measurable and locally bounded on $T(\mathcal{I})$. If direct inversion satisfies 
\begin{equation}
g\bigl(\mathbb{E}[T(Y)\mid x]\bigr)=\mathbb{E}[Y\mid x]
\end{equation}
for every conditional distribution of $Y\mid x$ supported on $\mathcal{I}$ with finite mean, then $g$ must be affine on every convex sub-interval of the support of $T(Y)$. 
\end{theorem}
This establishes that any nonlinear tail-compressing transform intrinsically loses universal expectation consistency, necessitating a separate recovery path.

\begin{theorem}[Bias of PIT-only inversion]
Let $C_0(y)=\Phi^{-1}(F_Y(y))$ be the unclipped, continuous PIT map and $\mu_C=\mathbb{E}[C_0(Y)\mid x]$. Under locally twice-differentiable $C_0^{-1}$, PIT-MSE yields $f^*(x)=\mu_C$. Direct inversion exhibits Taylor bias:
\begin{equation}
m(x)-C_0^{-1}(\mu_C)
=
\tfrac{1}{2}(C_0^{-1})''(\mu_C)
\operatorname{Var}\bigl(C_0(Y)\mid x\bigr)+R_3(x).
\end{equation}
\end{theorem}
The local curvature $(C_0^{-1})''$ dictates the systematic error direction. In our empirical setting, this bias is further compounded by implementation constraints, which we formalize below:

\begin{corollary}[Clipped Empirical PIT Error Decomposition]
For the clipped empirical map $C_\delta$ based on $\hat{F}_Y$, the inverse relation $C_\delta^{-1}(C_\delta(Y))=Y$ holds only in the interior. The total empirical inversion error bounds are bounded by the DKW uniform error $\eta$, amplified by local Lipschitz constants of the quantile function $L_Q$, plus boundary-clipped mass $\alpha_\delta(x)$ and tied-atom mass $\tau(x)$ (see Appendix A.2 and A.4).
\end{corollary}
These local constants become extreme near zero-spikes and heavy tails, rendering any direct inversion over a bare lookup table structurally unreliable.

\begin{theorem}[Fixed-Base SUN Recovery]
Assume $\operatorname{Var}(Y\mid x)<\infty$. If $b(x)>0$ is fixed conditional on $x$, the optimal SUN branch estimator exactly recovers the Bayes target:
\begin{equation}
z^*(x)=\frac{m(x)}{b(x)}, \qquad \hat{y}^*(x)=z^*(x)b(x)=m(x).
\end{equation}
\end{theorem}
In joint neural optimization, the stop-gradient operator enforcing $r=y/\operatorname{sg}[b(x)]$ preserves this fixed-base property for the recovery loss.

\begin{proposition}[Finite-Sample Conditioning via $b_{\min}$]
With the inverse-quantile base defined as $b(x)=\max(C^{-1}(\bar f(x)),b_{\min})+\varepsilon$, the recovery label $R_b=Y/b(x)$ satisfies:
\begin{equation}
\operatorname{Var}(R_b\mid x)
\le
\frac{\operatorname{Var}(Y\mid x)}{(b_{\min}+\varepsilon)^2}.
\end{equation}
Furthermore, output clipping enforces $|f(x)|\le a_f$ and $|C_\delta(Y)|\le a_\delta$, bounding each PIT-MSE sample gradient uniformly.
\end{proposition}
This demonstrates that $b_{\min}$ and clipping are not arbitrary heuristics, but strict variance controllers preventing denominator collapse in zero-inflated regions and gradient explosion in heavy tails.

\section{Experiments}
\label{sec:experiments}

\subsection{Experimental Setup}
\noindent\textbf{Datasets.} We evaluate PIT-SUN on 12 synthetic datasets, two public benchmarks (CIKM16, DTMart), one anonymous industrial benchmark (Indus), two large-scale real-world datasets (Dwell Time, Zero-inflated GMV), and an online incentive A/B test. Offline CDF and quantile tables are built only from training labels or validation-selected training windows; test labels are never used for lookup construction. Appendices C--D provide setup details and supplements.

\noindent\textbf{Baselines.} We compare with MSE, T-MSE (log, sqrt, $y^2$), manual-transform recovery methods (TranSUN/GTS), and specialized long-tailed or zero-inflated models (ZILN, WLR, MDME, TPM, CREAD, EGMN, GR, CCOR-Net). 

\noindent\textbf{Metrics.} We report point errors (NMAE, NRMSE), aggregate calibration (SRE, TRE, MRE, PGR), and ranking quality (xAUC, Normalized Gini, Spearman's $\rho$). TRE/MRE/PGR diagnose aggregate absolute error, instance-wise relative error, and signed prediction gain, respectively; xAUC is a pairwise ordering metric for continuous high-value targets. We report multi-seed means with paired tests ($p<0.05$).

\subsection{Main Results}

\subsubsection{Robustness on Synthetic Marginals}
\label{sec:synthetic}
We use 12 synthetic supervised-regression datasets with known conditional means, constructed to cover right-skewed heavy-tail (RS-*), left-skewed (LS-*), and symmetric (SM-*) marginals, together with zero inflation and multimodality. Figure~\ref{fig:synthetic_radar} gives a compact view of calibration robustness, while Appendix D reports the complete SRE values and multi-seed summaries. PIT-SUN shows the most stable profile without per-dataset transform selection, and remains consistently reliable as the target shape changes across tasks and noise families of very different forms.

\subsubsection{Generalization on Public Benchmarks}
Table~\ref{tab:public} evaluates whether PIT-SUN generalizes beyond the industrial setting on two public domains: D1 is long-tailed dwell time (CIKM16), and D2 is sparse transaction amount (DTMart). Dataset descriptions, train/test splits, and metric definitions are provided in Appendix C. Across point errors and ranking quality, PIT-SUN improves over distributional models (ZILN, EGMN), discretized/structured models (TPM, CCOR-Net), manual recovery baselines (TranSUN/GTS), and PIT plug-in variants (PIT-only and PIT-TranSUN). This indicates that the improvement is not only from empirical PIT, but from coupling the PIT coordinate with stabilized recovery.

\begin{figure}[t]
    \centering
    \includegraphics[width=0.85\columnwidth]{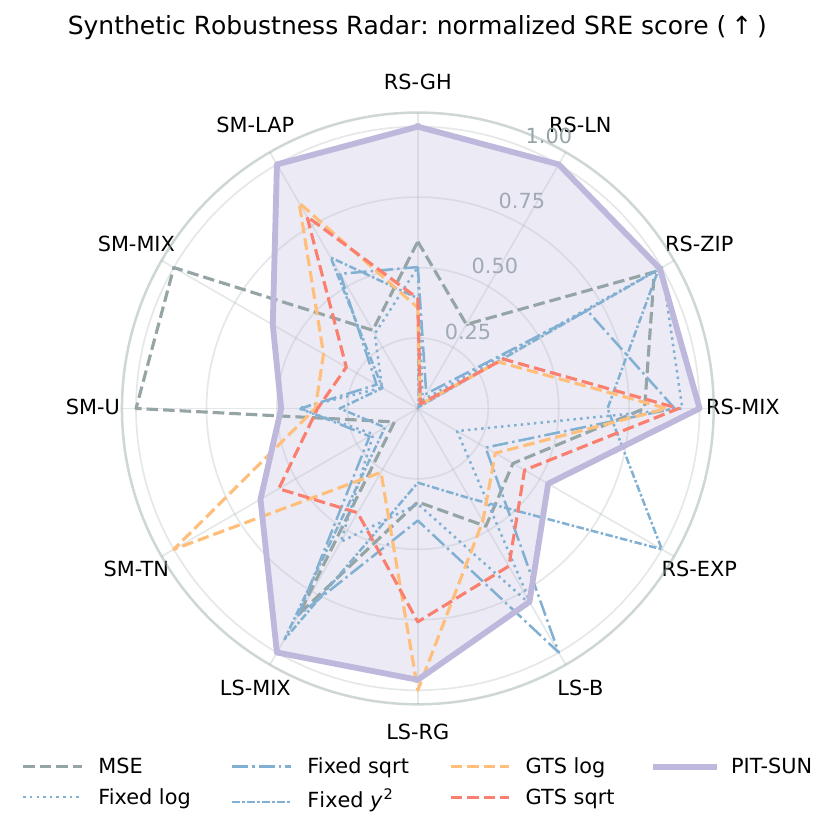}
    \caption{Synthetic robustness under normalized SRE. Larger values indicate better calibration.}
    \label{fig:synthetic_radar}
    \medskip
    \includegraphics[width=0.95\columnwidth]{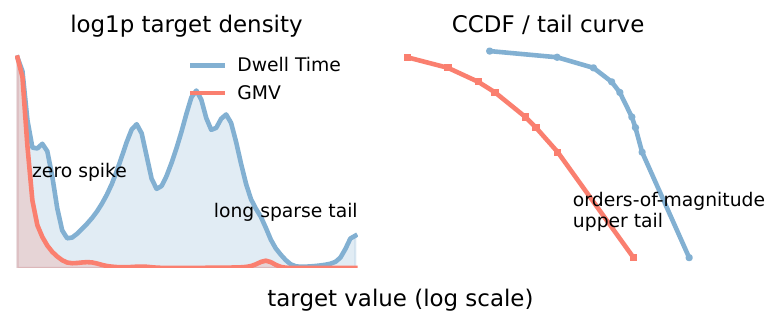}
    \caption{Anonymized industrial target diagnostics.}
    \label{fig:industrial_dist}
\end{figure}

\begin{table}[h]
\centering
{\fontsize{9pt}{10pt}\selectfont\setlength{\tabcolsep}{3pt}
\begin{tabular}{lcccccc}
\toprule
\multirow{2}{*}{Model} & \multicolumn{2}{c}{NRMSE $\downarrow$} & \multicolumn{2}{c}{NMAE $\downarrow$} & \multicolumn{2}{c}{xAUC $\uparrow$} \\
\cmidrule(lr){2-3} \cmidrule(lr){4-5} \cmidrule(lr){6-7}
 & D1 & D2 & D1 & D2 & D1 & D2 \\
\midrule
ZILN & 0.647 & 0.515 & 0.469 & \underline{0.221} & 0.681 & 0.935 \\
WLR & 0.872 & 1.019 & 0.614 & 0.579 & 0.681 & 0.932 \\
MDME & 1.580 & 0.973 & 1.148 & 0.346 & 0.678 & 0.936 \\
TPM & 0.571 & 0.509 & 0.441 & 0.223 & 0.683 & 0.935 \\
CREAD & 0.568 & 0.443 & 0.438 & 0.265 & 0.683 & 0.930 \\
EGMN & 0.595 & 0.493 & 0.474 & 0.301 & 0.687 & 0.916 \\
TranSUN & \underline{0.557} & 0.449 & \underline{0.437} & 0.225 & 0.678 &0.936 \\
GR & 0.601 & 0.452 & 0.482 & 0.265 & \underline{0.688} & 0.927 \\
CCOR-Net & 0.574 & \underline{0.426} & 0.449 & 0.275 & 0.680 & 0.930 \\
PIT-only & 0.612 & 0.461 & 0.479 & 0.244 & 0.672 & 0.922 \\
PIT-TranSUN & 0.563 & 0.436 & 0.447 & 0.238 & 0.682 & \underline{0.937} \\
\midrule
\textbf{PIT-SUN} & \textbf{0.547} & \textbf{0.412} & \textbf{0.424} & \textbf{0.207} & \textbf{0.695} & \textbf{0.943} \\
\bottomrule
\end{tabular}
}
\caption{Public benchmark results ($p<0.05$). D1 is CIKM16 and D2 is DTMart. Best results are in bold, second-best are underlined.}
\label{tab:public}
\end{table}

\subsubsection{Large-Scale Industrial Deployments}
Before presenting results, we link the industrial domains to the regression dilemma. Figure~\ref{fig:industrial_dist} visualizes the anonymized marginal diagnostics used for this motivation: Dwell Time is long-tailed, while GMV is sparser and more tail-concentrated. These regimes stress original-space MSE and fixed transforms in different ways, motivating an empirical recovery design that is tested below at deployment scale.

\begin{table*}[h]
\centering
{\fontsize{9pt}{10pt}\selectfont\setlength{\tabcolsep}{3pt}
\begin{tabular}{lcccccccc}
\toprule
Model & NMAE $\downarrow$ & NRMSE $\downarrow$ & TRE $\downarrow$ & MRE $\downarrow$ & PGR $\to 0$ & xAUC $\uparrow$ & Norm\_Gini $\uparrow$ & Spear$\rho$ $\uparrow$ \\
\midrule
T-MSE(ln) & 0.696 & 1.343 & 0.338 & 2.351 & -0.193 & 0.827 & 0.823 & 0.783 \\
T-MSE(sqrt) & 0.682 & 1.292 & 0.324 & 2.356 & \underline{-0.090} & 0.820 & 0.822 & 0.778 \\
TranSUN(ln) & 0.575 & 1.262 & 0.183 & 1.385 & 0.183 & 0.882 & 0.835 & 0.793 \\
TranSUN(sqrt) & 0.622 & 1.158 & 0.207 & 1.453 & 0.207 & 0.894 & 0.831 & 0.798 \\
GTS(ln) & 0.570 & 1.060 & 0.193 & 1.538 & 0.109 & 0.867 & 0.856 & 0.821 \\
GTS(sqrt) & 0.504 & 0.988 & 0.149 & 1.253 & 0.099 & 0.878 & 0.858 & 0.823 \\
ZILN & 1.094 & 1.522 & 0.369 & 1.727 & 0.417 & 0.842 & 0.859 & 0.784 \\
WLR & 0.832 & 1.394 & 0.224 & 1.206 & 0.353 & 0.824 & 0.801 & 0.807 \\
TPM & 0.641 & 1.233 & 0.265 & 1.808 & -0.140 & 0.864 & 0.853 & 0.812 \\
CREAD & 0.531 & \underline{0.926} & 0.097 & \underline{0.878} & 0.097 & 0.904 & \underline{0.865} & \underline{0.843} \\
EGMN & 0.633 & 1.125 & 0.229 & 1.034 & -0.275 & 0.875 & 0.828 & 0.815 \\
GR & 0.688 & 1.311 & 0.177 & 1.288 & -0.269 & 0.896 & 0.812 & 0.801 \\
CCOR-Net & \underline{0.487} & 0.927 & 0.166 & 1.198 & -0.166 & 0.893 & 0.838 & 0.837 \\
PIT-TranSUN & 0.801 & 1.086 & \underline{0.083} & 1.030 & -0.168 & \underline{0.910} & 0.837 & 0.808 \\
\midrule
\textbf{PIT-SUN} & \textbf{0.477} & \textbf{0.885} & \textbf{0.051} & \textbf{0.818} & \textbf{0.054} & \textbf{0.921} & \textbf{0.878} & \textbf{0.852} \\
\midrule
\rowcolor{gray!10} PIT-SUN-w/o SUN & 0.496 & 0.905 & 0.054 & 1.025 & -0.062 & 0.918 & 0.872 & 0.835 \\
\rowcolor{gray!10} PIT-SUN-w/o $\mathit{sg}$ & 0.524 & 0.908 & 0.060 & 0.942 & 0.099 & 0.917 & 0.843 & 0.817 \\
\rowcolor{gray!10} PIT-SUN w/ rand-rank & 0.545 & 0.972 & 0.062 & 0.997 & 0.060 & 0.913 & 0.852 & 0.810 \\
\bottomrule
\end{tabular}
}
\caption{Large-scale Industrial Dwell Time Forecasting ($p<0.05$). PIT-TranSUN is a strong plug-in baseline that replaces the manual transform in TranSUN/GTS with empirical PIT and uses $C^{-1}(\bar f(x))+\varepsilon$ as the recovery base without the $b_{\min}$ lower-bound stabilization. Shaded rows denote diagnostic variants.}
\label{tab:dwell}
\end{table*}
Tables~\ref{tab:dwell} and \ref{tab:gmv} report the offline results. PIT-TranSUN is included as a plug-in test: it substitutes empirical PIT into the TranSUN template while leaving the recovery base as a bare inverse lookup. Its weaker point-error and sparse-value behavior indicate that the gain comes from the full PIT-SUN closure---bounded coordinate, stabilized inverse-quantile base, and conditionally linear recovery---rather than from PIT alone. Overall, PIT-SUN improves point accuracy, aggregate calibration, and ranking while avoiding the signed PGR biases of fixed transforms and the manual-base dependence of recovery baselines.

GMV further separates occurrence modeling from value estimation. We report Zero-AUC and positive recall in Table~\ref{tab:gmv}: PIT-SUN remains the strongest end-to-end value estimator, while PIT-SUN-ZI improves occurrence diagnostics and zero-positive decoupling. PIT-SUN-ZI is preferred when zero-occurrence monitoring is required, whereas PIT-SUN is the default when value estimation and calibration dominate. Bucket-level, zero-positive, sensitivity, and drift diagnostics are deferred to Appendix D.

\begin{table*}[h]
\centering
{\fontsize{9pt}{10pt}\selectfont\setlength{\tabcolsep}{3pt}
\begin{tabular}{lcccccccccc}
\toprule
Model & NMAE $\downarrow$ & NRMSE $\downarrow$ & TRE $\downarrow$ & MRE $\downarrow$ & PGR $\to 0$ & xAUC $\uparrow$ & Norm\_Gini $\uparrow$ & Spear$\rho$ $\uparrow$ & Zero-AUC $\uparrow$ & Pos. Recall $\uparrow$ \\
\midrule
T-MSE(ln) & 1.249 & 6.541 & 0.544 & 2.152 & 0.544 & 0.829 & 0.846 & 0.775 & 0.812 & 0.608 \\
T-MSE(sqrt) & 1.361 & 5.059 & 0.428 & 2.319 & -0.228 & 0.818 & 0.844 & 0.754 & 0.819 & 0.614 \\
TranSUN(ln) & 0.748 & 5.091 & 0.429 & 1.914 & 0.096 & 0.877 & 0.860 & 0.760 & 0.851 & 0.632 \\
TranSUN(sqrt) & 0.744 & 5.613 & 0.271 & 2.087 & -0.063 & 0.865 & 0.886 & 0.775 & 0.858 & 0.641 \\
GTS(ln) & 0.735 & 4.185 & 0.328 & 1.411 & \underline{-0.043} & 0.862 & 0.875 & 0.772 & 0.862 & 0.646 \\
GTS(sqrt) & 0.719 & 4.431 & 0.172 & 1.387 & -0.072 & 0.871 & 0.873 & 0.768 & 0.866 & 0.652 \\
ZILN & 1.537 & 6.194 & 0.571 & 1.441 & 0.683 & 0.509 & 0.223 & 0.782 & 0.801 & 0.587 \\
WLR & 1.479 & 6.313 & 0.615 & 2.350 & 1.837 & 0.836 & 0.846 & 0.645 & 0.836 & 0.604 \\
MDME & 0.782 & 5.240 & 0.374 & 1.645 & 0.096 & 0.868 & 0.868 & 0.766 & 0.861 & 0.638 \\
TPM & 0.767 & 5.643 & 0.427 & 1.888 & 0.473 & 0.873 & 0.869 & 0.771 & 0.865 & 0.644 \\
CREAD & \underline{0.646} & 3.988 & 0.185 & 1.324 & -0.174 & 0.893 & 0.859 & 0.788 & 0.871 & 0.641 \\
EGMN & 0.716 & 4.693 & 0.311 & 1.584 & -0.154 & 0.854 & 0.873 & 0.731 & 0.858 & 0.633 \\
GR & 0.733 & 4.465 & 0.354 & 2.049 & -0.357 & 0.898 & 0.874 & 0.763 & 0.863 & 0.649 \\
CCOR-Net & 0.658 & 4.830 & 0.192 & 1.879 & -0.192 & 0.861 & 0.871 & 0.770 & 0.867 & 0.655 \\
PIT-TranSUN & 1.009 & 4.331 & 0.298 & 1.623 & -0.095 & 0.872 & 0.860 & 0.772 & 0.861 & 0.626 \\
\midrule
\textbf{PIT-SUN} & \textbf{0.623} & \textbf{3.853} & \textbf{0.071} & \textbf{1.143} & \textbf{0.027} & \textbf{0.909} & \textbf{0.902} & \textbf{0.808} & \underline{0.884} & \underline{0.663} \\
\textbf{PIT-SUN-ZI} & 0.719 & \underline{3.973} & \underline{0.098} & \underline{1.279} & -0.098 & \underline{0.906} & \underline{0.898} & \underline{0.807} & \textbf{0.902} & \textbf{0.681} \\
\midrule
\rowcolor{gray!10} PIT-SUN-w/o SUN & 0.688 & 4.118 & 0.347 & 1.673 & -0.120 & 0.884 & 0.875 & 0.785 & 0.872 & 0.651 \\
\rowcolor{gray!10} ZI w/o Amt-SUN & 0.689 & 4.272 & 0.250 & 1.886 & -0.194 & 0.849 & 0.878 & 0.781 & 0.897 & 0.676 \\
\bottomrule
\end{tabular}
}
\caption{Large-scale Industrial Zero-inflated GMV Forecasting ($p<0.05$). PIT-SUN-ZI denotes a zero-inflated extension for sparse value domains; PIT-TranSUN is the empirical-PIT plug-in recovery baseline using the same inverse lookup but without the $b_{\min}$ floor; shaded rows denote diagnostic variants.}
\label{tab:gmv}
\end{table*}
\begin{table*}[h]
\centering
{\fontsize{9pt}{10pt}\selectfont\setlength{\tabcolsep}{3pt}
\begin{tabular}{lccccc}
\toprule
Strategy (Reward) & DAU & Watching Time & Video View Cnt & Ad Revenue & Deeply Engaged Users \\
\midrule
Single-Obj (Time) & 0.010\% & 0.318\%$^*$ & 0.265\%$^*$ & 0.148\% & 0.195\%$^*$ \\
\textcolor{gray}{95\% CI} & \textcolor{gray}{[-0.01\%,0.03\%]} & \textcolor{gray}{[0.28\%,0.36\%]} & \textcolor{gray}{[0.20\%,0.33\%]} & \textcolor{gray}{[-0.06\%,0.36\%]} & \textcolor{gray}{[0.16\%,0.23\%]} \\
Multi-Obj (Time, Revenue, etc.) & 0.003\% & 0.330\%$^*$ & 0.273\%$^*$ & 0.216\% & 0.064\%$^*$ \\
\textcolor{gray}{95\% CI} & \textcolor{gray}{[-0.03\%,0.04\%]} & \textcolor{gray}{[0.26\%,0.40\%]} & \textcolor{gray}{[0.16\%,0.39\%]} & \textcolor{gray}{[-0.08\%,0.52\%]} & \textcolor{gray}{[0.03\%,0.10\%]} \\
\bottomrule
\end{tabular}
}
\caption{7-day average relative lift in online A/B testing. $^*$ indicates $p<0.05$ via user-level bootstrap-based A/B test. Metrics without $^*$ have confidence intervals that include zero and are interpreted as stable or directional rather than significant gains.}
\label{tab:online_ab}
\end{table*}

\subsection{Component Diagnostics}
The main tables establish the empirical comparison; this subsection checks whether the proposed components play the roles required by the method design. We use diagnostic variants to test four links in the PIT-SUN chain: a bounded coordinate alone is insufficient for expectation recovery; a PIT plug-in base still needs denominator stabilization; joint training needs gradient isolation to preserve branch roles; and zero-positive factorization is useful only when occurrence monitoring is part of the deployment objective.

\noindent\textbf{Coordinate without recovery.} PIT-SUN-w/o SUN keeps the empirical PIT coordinate but removes the conditionally linear recovery path. This variant is a diagnostic for the PIT-only bias result: bounded labels help representation learning, but direct inverse-style prediction does not by itself define the original-space conditional mean.

\noindent\textbf{PIT plug-in without base stabilization.} PIT-TranSUN sets $T(y)=C(y)$ in the TranSUN/GTS template and uses $C^{-1}(\bar f(x))+\varepsilon$ as the base. This isolates the question of whether empirical PIT can simply replace a manual transform. Its behavior shows why PIT-SUN treats the CDF table as a full closure object: coordinate, inverse lookup, lower-bounded base, and recovery must be specified together.

\noindent\textbf{Gradient isolation.} PIT-SUN-w/o $\mathit{sg}$ allows recovery residuals to update the inverse-quantile path. This diagnostic tests the fixed-base interpretation of SUN recovery rather than adding another baseline: PIT should shape the bounded coordinate, while SUN should fit the expectation ratio under an isolated base.

\noindent\textbf{Zero-positive factorization.} PIT-SUN-ZI is included to clarify the boundary of the single-track estimator. It is useful when zero-occurrence diagnostics are operationally important, whereas the positive-amount branch still requires SUN recovery for value calibration, as shown by ZI w/o Amt-SUN.

\subsection{Large-Scale Online A/B Testing}
To evaluate end-to-end business value, we deployed PIT-SUN in an online incentive allocation system. The model estimates complex targets to guide budget-constrained distribution strategies. Online streaming data is long-tailed, sparse, and noisy, requiring stable regression.

\begin{figure}[h]
    \centering
    \includegraphics[width=0.5\columnwidth]{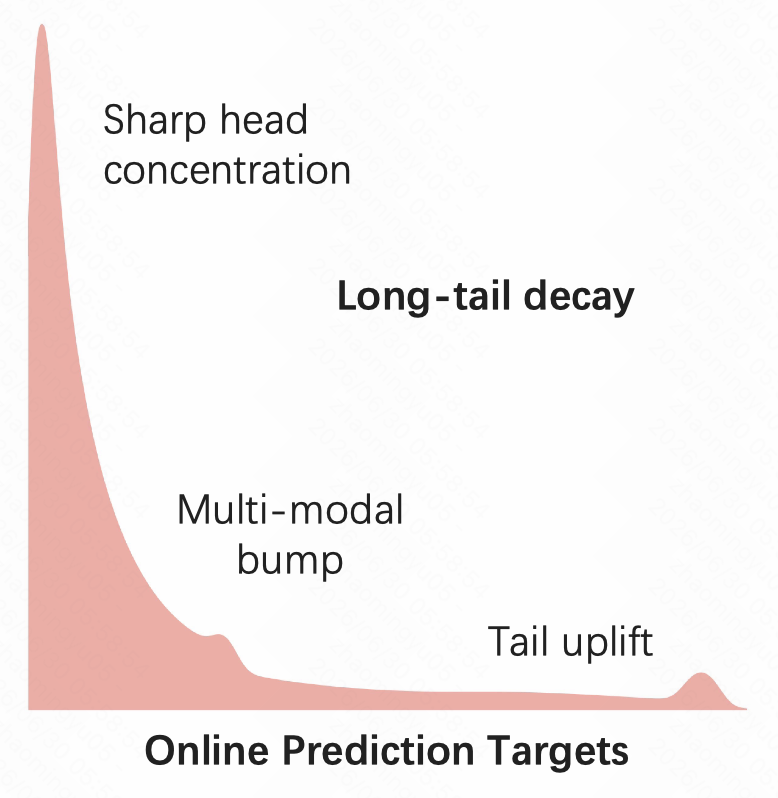}
    \caption{Online target distribution in the deployed incentive allocation system.}
    \label{fig:online_dist}
\end{figure}

Table~\ref{tab:online_ab} reports average relative lifts over a continuous 7-day online A/B test, with significance from user-level bootstrap testing. We deployed two PIT-SUN-guided reward configurations: \textbf{Dwell Time} and a \textbf{Multi-Objective Mixture} (dwell time, engagement depth, and commercial value). Due to compliance constraints, Table~\ref{tab:online_ab} reports relative lifts and confidence intervals rather than raw traffic counts; Appendix C.4 describes the randomized bucketization, fixed production controls, and leakage checks used to isolate the estimator change. The results show statistically significant gains in dwell time, video view counts, and deeply engaged users. DAU and ad revenue
remains statistically stable. We therefore interpret the online evidence as significant engagement improvement with neutral guardrail movement and directionally positive commercial signals, rather than as a statistically confirmed ad-revenue gain.

\begin{table}[h]
\centering
{\fontsize{9pt}{10pt}\selectfont\setlength{\tabcolsep}{3pt}
\begin{tabular}{lccc}
\toprule
Model & Training Time & Inf. P50 & Inf. P99 \\
\midrule
T-MSE(with Expect-aligned) & 40min 04s & 17.93 ms & 23.67 ms \\
Two-stage & 55min 17s & 20.01 ms & 25.17 ms \\
PIT-SUN (Multi-obj) & 36min 24s & 14.61 ms & 18.35 ms \\
PIT-SUN (Single-obj)& 31min 33s & 15.62 ms & 20.05 ms \\
\bottomrule
\end{tabular}
}
\caption{Online training and inference overhead comparison.}
\label{tab:overhead}
\end{table}

Table~\ref{tab:overhead} profiles deployment overhead. PIT-SUN shortens training time and reduces inference latency versus existing baselines, supporting high-QPS deployment.

\section{Conclusion}

PIT-SUN provides a data-driven closure for the stability--expectation dilemma in recommender-system regression. This positioning reframes target transformation from heuristic preprocessing into a deployable empirical-marginal closure for value-oriented recommendation. Across synthetic, public, industrial offline, and online settings, PIT-SUN shows robust calibration and ranking behavior under complex marginals, with stable advantages across both sparse and heavy-tailed regimes. The appendix further supports the design with PIT-TranSUN plug-in comparisons, recovery-base diagnostics, CDF drift analysis, bucket-level calibration, and subgroup-heterogeneity diagnostics, clarifying why the complete empirical-marginal closure is necessary: PIT alone lacks expectation recovery, bare inverse lookup leaves recovery labels ill-conditioned, and zero-positive decoupling still requires SUN for calibrated positive amounts. Beyond individual regression tasks, PIT-SUN also suggests a practical direction for multi-task learning: by mapping heterogeneous continuous targets such as dwell time, GMV, and LTV into comparable normal-score coordinates before expectation recovery, it can serve as an interpretable and model-agnostic scale-aligner for industrial recommendation systems.

\bibliography{aaai2027}


\section{Appendix A: Proofs for the Theoretical Justification}
\label{app:proofs}
Appendix A provides the mathematical details supporting the Theoretical Justification section. The proofs follow the same design chain as the main text: nonlinear direct inversion cannot be both stabilizing and universally expectation-consistent; PIT creates a bounded learning coordinate but direct PIT inversion remains biased; SUN recovery restores the original-space expectation under a fixed positive base; and clipped PIT together with the lower-bounded inverse-quantile base improves the conditioning of the main and recovery subproblems. The results are population or local finite-sample conditioning statements, not distribution-free deployment guarantees under arbitrary drift. Empirical CDF error, boundary clipping, ties, and refresh policies are handled by generalized inverses, explicit clipping constants, local quantile regularity, and the diagnostics in Appendix D. The comparison table then clarifies how PIT-SUN differs from transformed MSE and TranSUN/GTS.

\subsection{A.1 Direct Inversion Requires an Affine Inverse}
This subsection proves the structural limitation behind target transformation. Let $T$ be strictly monotone on an interval $\mathcal{I}$ containing the support of the conditional target, and let $g=T^{-1}$ be measurable, locally bounded, and integrable under the transformed distributions considered below. Suppose direct inversion were universally expectation-consistent:
\begin{equation}
g\bigl(\mathbb{E}[T(Y)\mid x]\bigr)=\mathbb{E}[Y\mid x]
\end{equation}
for every conditional distribution of $Y\mid x$ with finite mean. Since $Y=g(W)$ for $W=T(Y)$, this condition is equivalent to
\begin{equation}
g\bigl(\mathbb{E}[W\mid x]\bigr)=\mathbb{E}[g(W)\mid x]
\end{equation}
for every distribution of $W\mid x$ supported on the transformed domain with finite first moment and finite $\mathbb{E}|g(W)|$. Consider any two points $w_1,w_2$ in a convex sub-interval of this domain and a Bernoulli mixture $W=w_1$ with probability $p$ and $W=w_2$ with probability $1-p$. The condition gives
\begin{equation}
g(pw_1+(1-p)w_2)=p g(w_1)+(1-p)g(w_2)
\end{equation}
for all $p\in[0,1]$. Thus $g$ preserves every chord and is Jensen-affine on the interval. A measurable or locally bounded Jensen-affine function is affine; monotone inverse transforms used in regression satisfy this regularity automatically on their effective support. The statement applies to log/sqrt/Box-Cox after restricting to their non-negative support and to PIT-type maps after using generalized inverse intervals for atoms. Consequently, any nonlinear inverse transform that compresses heavy tails must fail for some conditional distribution. On a fixed empirical task, however, the relevant diagnostic is finite: compare $\tilde y_{\mathrm{PIT}}=C^{-1}(f(x))$ against original-space labels through bias, SRE/TRE, and bucket calibration, as done for PIT-only versus PIT-SUN in Appendix D.6. The familiar retransformation bias is therefore not a removable artifact of log, sqrt, or PIT; it is the reason a separate conditionally linear recovery path is required.

\subsection{A.2 PIT-only Inversion Is Biased}
The PIT-only bias statement is stated for the ideal unclipped, continuous, invertible PIT map. Clipping is used in implementation for finite-sample conditioning and adds a boundary truncation error controlled by $\delta$. The PIT branch minimizes
\begin{equation}
\mathcal{L}_{\mathrm{PIT}}
=\mathbb{E}_{(x,y)\sim\mathcal{D}}
\bigl[(f(x;\theta)-C(y))^2\bigr].
\end{equation}
For each fixed $x$, the population minimizer satisfies
\begin{equation}
f^*(x)
=\arg\min_u\mathbb{E}\bigl[(u-C(Y))^2\mid X=x\bigr]
=\mathbb{E}[C(Y)\mid x].
\end{equation}
If one directly inverts this prediction, the estimator becomes
\begin{equation}
\tilde{y}_{\mathrm{PIT}}(x)
=C^{-1}\bigl(\mathbb{E}[C(Y)\mid x]\bigr),
\end{equation}
which generally differs from
\begin{equation}
\mathbb{E}[C^{-1}(C(Y))\mid x]=m(x).
\end{equation}
Let $W=C(Y)$, $\mu_C(x)=\mathbb{E}[W\mid x]$, and $g=C^{-1}$. A second-order Taylor expansion around $\mu_C(x)$ gives
\begin{equation}
m(x)=g(\mu_C(x))
+\tfrac{1}{2}g''(\mu_C(x))\operatorname{Var}(W\mid x)+R_3(x),
\end{equation}
where $R_3(x)=o(\operatorname{Var}(W\mid x))$ when the conditional mass is concentrated in a neighborhood where $g''$ is continuous; the sign of $g''$ determines whether direct PIT inversion under- or over-estimates locally.
Thus PIT alone provides a stable coordinate but not an unbiased original-space expectation estimator. For the clipped empirical map $C_\delta$, $C_\delta^{-1}(C_\delta(Y))=Y$ holds only away from clipped boundaries and flat atoms. The following corollary makes the implementation-level terms explicit.

\begin{corollary}[Clipped Empirical PIT Error Decomposition]
Let $C_\delta$ be the clipped empirical PIT map implemented with a generalized inverse, and define
\[
\mu_\delta(x)=\mathbb{E}[C_\delta(Y)\mid x],\qquad
\tilde y_{\mathrm{PIT}}(x)=C_\delta^{-1}(\mu_\delta(x)).
\]
On the interior non-atomic neighborhood of $\mu_\delta(x)$, let $\kappa_\delta(x)=(C_\delta^{-1})''(\mu_\delta(x))$, and let $L_Q(x)$ be a local upper bound on the inverse-quantile slope. Let $\alpha_\delta(x)$ be the conditional mass clipped to the two PIT boundaries, $\tau(x)$ the conditional mass assigned to tied atoms, and $\eta_N=\sup_y|\widehat F_Y(y)-F_Y(y)|$ the empirical CDF error. Then the clipped empirical direct-inversion gap satisfies the local bound
\begin{equation}
\begin{aligned}
&\left|
m(x)-\tilde y_{\mathrm{PIT}}(x)
-\tfrac{1}{2}\kappa_\delta(x)
\operatorname{Var}(C_\delta(Y)\mid x)
\right| \\
&\qquad\le
|R_{3,\delta}(x)|
+L_Q(x)\{\alpha_\delta(x)+\tau(x)+\eta_N\}.
\end{aligned}
\tag{14a}
\end{equation}
\end{corollary}
The bound shows that boundary samples contribute through the clipped mass $\alpha_\delta(x)$, tied atoms contribute through the atom mass $\tau(x)$ and generalized-inverse intervals, and empirical-table estimation contributes through the CDF error $\eta_N$. These effects do not invalidate SUN recovery, because the final estimator is not the direct inverse $C_\delta^{-1}(f(x))$; they delimit where the second-order bias approximation is informative. This is the formal reason the SUN recovery branch is necessary.

\noindent\textit{Proof sketch for Corollary 1.}
The ideal expansion gives the curvature term in Eq.~(14). Clipping replaces tail PIT values by boundary constants, so the affected conditional mass contributes at most a local inverse-quantile slope times the clipped mass, yielding $L_Q(x)\alpha_\delta(x)$. Tied atoms replace the smooth inverse by a generalized-inverse interval; the corresponding interval displacement is bounded by the same local slope times the tied mass, yielding $L_Q(x)\tau(x)$. Finally, replacing $F_Y$ by $\widehat F_Y$ perturbs quantile levels by at most $\eta_N$, which gives $L_Q(x)\eta_N$ away from the active floor and boundary. Combining these terms with the local Taylor remainder gives Eq.~(14a).

\subsection{A.3 SUN Recovery Restores Expectation Consistency}
The recovery base is
\begin{equation}
\label{eq:app_base}
b(x)=\max\bigl(C^{-1}(\operatorname{clip}(f(x;\theta),-a_\delta,a_\delta)),
 b_{\min}\bigr)+\varepsilon .
\end{equation}
Under the fixed-base subproblem, stop-gradient makes $b(x)$ an $x$-dependent positive constant. The recovery objective is therefore minimized by
\begin{align}
z^*(x)
&=\arg\min_v\mathbb{E}\left[
\left(v-\frac{Y}{b(x)}\right)^2\middle|x\right] \\
&=\mathbb{E}\left[\frac{Y}{b(x)}\middle|x\right]
=\frac{m(x)}{b(x)} .
\end{align}
The multiplicative prediction exactly recovers the Bayes target:
\begin{equation}
\hat{y}^*(x)=z^*(x)b(x)=m(x).
\end{equation}
This result establishes the recovery subproblem as an expectation-preserving operator. In joint training with shared bottom layers, stop-gradient maintains this operator by forming the recovery label as $Y/\operatorname{sg}[b(x)]$: the recovery head receives gradients, but the proportional-label path does not update the PIT coordinate, clipping operator, inverse-quantile lookup, or $b_{\min}$ floor.

\paragraph{Joint-training approximation.}
The fixed-base result describes the population optimum of the SUN subproblem. In joint neural training, the base $b_t(x)$ may change between optimization steps through the PIT branch and shared representation. Let
\[
\Delta b_t(x)=b_{t+1}(x)-b_t(x)
\]
be the adjacent-step base drift, and suppose the ratio head is locally bounded by $|z_t(x)|\le M_z$. If the recovery head update is evaluated against the previous base $b_t(x)$, the multiplicative prediction perturbation induced by base drift satisfies
\begin{equation}
\left|z_t(x)b_{t+1}(x)-z_t(x)b_t(x)\right|
\le M_z|\Delta b_t(x)|.
\end{equation}
Averaging over the data distribution gives the empirical proxy
\begin{equation}
\mathcal{E}_{\mathrm{drift}}
=
\mathbb{E}_x\left[|z_t(x)|\,|\Delta b_t(x)|\right],
\end{equation}
which measures the gap between the fixed-base recovery interpretation and the joint-training trajectory. We also monitor a gradient-interference proxy
\begin{equation}
\rho_{\mathrm{PIT}}
=
\frac{
\|\nabla_{\theta_{\mathrm{PIT}}}\mathcal{L}_{\mathrm{SUN}}\|
}{
\|\nabla_{\theta_{\mathrm{PIT}}}\mathcal{L}_{\mathrm{PIT}}\|
+
\|\nabla_{\theta_{\mathrm{PIT}}}\mathcal{L}_{\mathrm{SUN}}\|
+\varepsilon
},
\end{equation}
where $\theta_{\mathrm{PIT}}$ denotes the parameters controlling the PIT coordinate and inverse-quantile base. A smaller $\rho_{\mathrm{PIT}}$ indicates that recovery residuals do not dominate the geometry learned by the PIT branch. Table~\ref{tab:two_stage} reports adjacent-epoch $\Delta b$ percentiles and $\rho_{\mathrm{PIT}}$, verifying that joint training with stop-gradient stays close to the fixed-base interpretation while avoiding the cost of a two-stage pipeline.

\subsection{A.4 Base Quality, Bounded Constants, and Empirical Quantiles}
The previous subsection shows that any positive fixed base preserves expectation consistency at the population optimum. The remaining question is why the proposed base is a stable choice for finite-sample learning. Let $R_b=Y/b(x)$ denote the recovery label. For a fixed $x$,
\begin{equation}
\operatorname{Var}(R_b\mid x)=\frac{\operatorname{Var}(Y\mid x)}{b(x)^2},
\end{equation}
so very small bases inflate the regression noise seen by the SUN branch. Conversely, if $b(x)$ is much larger than the local scale of $Y\mid x$, the optimal ratio $m(x)/b(x)$ becomes unnecessarily small and may underuse a bounded prediction range. If the ratio head is effectively constrained to $z(x)\in[0,M]$, then exact recovery of $m(x)$ requires $b(x)\ge m(x)/M$, while variance control favors bases that are not too small. This motivates using a scale proxy that tracks the predicted PIT location rather than a global constant.

The inverse-quantile base $C^{-1}(\bar f(x))$ provides such a proxy: it maps the learned marginal rank back to an original-scale value, thereby adapting the denominator to the region of the empirical label distribution where $x$ is predicted to lie. This makes the recovery base compatible with the PIT coordinate rather than an independently tuned transform. The floor $b_{\min}$ then handles the failure mode of sparse zero plateaus by enforcing
\begin{equation}
b(x)\ge b_{\min}+\varepsilon,
\end{equation}
and therefore
\begin{equation}
\operatorname{Var}\left(\frac{Y}{b(x)}\middle|x\right)
\le
\frac{\operatorname{Var}(Y\mid x)}{(b_{\min}+\varepsilon)^2}.
\end{equation}
This bound shows why $b_{\min}$ is part of the recovery architecture rather than another manually selected target transform: it clips the variance of recovery labels while leaving the PIT representation geometry unchanged. The bound is loose because it uses only the global floor and ignores how $C^{-1}(\bar f(x))$ tracks local rank. Therefore it should not be read as an analytical formula for the optimal floor. In practice, Appendix D selects $q_b$ by jointly monitoring ratio variance, high-percentile ratios, floor-active fraction, and NMAE/TRE; this connects Eq.~(\ref{eq:app_base}) to the clipped-fraction columns in Tables~\ref{tab:bmin_sens_dur}--\ref{tab:bmin_sens_gmv}.

Clipping enforces $|C_\delta(Y)|\le a_\delta$, so if $|f_\theta(x)|\le a_f$ after output clipping and $\|\nabla_\theta f_\theta(x)\|\le G_f$ (or their empirical high-percentile training-log estimates), the PIT-MSE gradient obeys
\begin{equation}
\left\|\nabla_\theta \ell_{\mathrm{PIT}}(x,y)\right\|
\le 2(a_f+a_\delta)G_f .
\end{equation}
It also bounds total and average conditional variance:
\begin{equation}
\operatorname{Var}(C_\delta(Y))\le a_\delta^2,
\quad
\mathbb{E}_X[\operatorname{Var}(C_\delta(Y)\mid X)]\le a_\delta^2.
\end{equation}
For the empirical CDF, the DKW inequality gives, with probability at least $1-2e^{-2N\eta^2}$,
\begin{equation}
\sup_y |\hat{F}_Y(y)-F_Y(y)|\le \eta .
\end{equation}
After clipping away from the boundary, the normal-score labels inherit a controlled uniform error proportional to $\eta/\phi(a_\delta)$. For a fixed PIT prediction $w\in[-a_\delta,a_\delta]$, if the population quantile $Q(u)=F_Y^{-1}(u)$ is locally Lipschitz with constant $L_Q$ around $u=\Phi(w)$, the induced base error satisfies
\begin{equation}
|\hat{b}(x)-b(x)|
\le L_Q\eta
\end{equation}
away from the active $b_{\min}$ floor, and is further truncated when the floor is active. Consequently the final multiplicative estimator obeys the local perturbation relation
\begin{equation}
|\hat{z}(x)\hat{b}(x)-z(x)b(x)|
\le |\hat{z}(x)-z(x)|b(x)+|\hat{z}(x)|L_Q\eta,
\end{equation}
which makes explicit how empirical CDF error propagates through the base rather than implying a distribution-free guarantee under arbitrary drift. The slope $L_Q$ is estimated in deployment by adjacent quantile-table differences within monitored buckets; large slopes near zero spikes or upper tails are precisely where $b_{\min}$ and refresh diagnostics matter. This local relation underlies Appendix D.5: quantile KS proxies $\eta$, and refresh is triggered when quantile discrepancy or bucket-level calibration drift exceeds the validation-calibrated band. In our industrial diagnostic, sliding-window KS stays near $0.04$, while stale KS reaches $0.118$ with worse worst-block TRE.

\subsection{A.5 Estimator Comparison}
Table~\ref{tab:estimators} summarizes the estimator-level distinctions. Transformed MSE and PIT-MSE have stable learning targets but biased direct inversions; by the affine-inverse argument, this is unavoidable for any nonlinear stabilizing coordinate. TranSUN/GTS restore expectation consistency through conditional linearity, but their constants depend on manually selected transforms. PIT-SUN keeps the same recovery principle while deriving the coordinate and base from the empirical marginal distribution.

\begin{table}[h]
\centering
{\fontsize{9pt}{10pt}\selectfont\setlength{\tabcolsep}{3pt}
\begin{tabular}{lll}
\toprule
Method & Estimator & Issue \\
\midrule
T-MSE & $T^{-1}(\mathbb{E}[T(Y)|x])$ & Bias \\
PIT-MSE & $C^{-1}(\mathbb{E}[C(Y)|x])$ & Bias \\
TranSUN/GTS & Recovery with manual $T$ & Manual choice \\
\textbf{PIT-SUN} & Empirical PIT + SUN & CDF refresh \\
\bottomrule
\end{tabular}
}
\caption{Estimator-level comparison.}
\label{tab:estimators}
\end{table}

\section{Appendix B: FAQ and Extended Discussions}
\label{app:design}
Appendix B keeps the FAQ form because these questions match common concerns when applying PIT-SUN in industrial recommender systems. The answers focus on what each design choice solves: preserving the MSE Bayes target, controlling finite-sample constants, handling sparse empirical CDFs, and clarifying deployment extensions.

\subsection{B.1 Does PIT-SUN replace original-space MSE's Bayes target?}
No. Original-space MSE already targets $m(x)$, and PIT-SUN does not introduce a new Bayes target. The problem lies in finite-sample optimization under complex marginals: heavy tails inflate gradients of the form $2(g(x)-y)\nabla g(x)$, exposing training to unbounded target amplitudes.

PIT-SUN reparameterizes the learning route while keeping the same asymptotic point target. PIT supplies a bounded coordinate for representation learning, and SUN recovery returns the estimator to $m(x)$. Table~\ref{tab:mse_relation} summarizes this relationship.

\begin{table}[h]
\centering
{\fontsize{9pt}{10pt}\selectfont\setlength{\tabcolsep}{3pt}
\begin{tabular}{lcc}
\toprule
Method & Point target & Finite-sample behavior \\
\midrule
Orig. MSE & $m(x)$ & Exposed to extreme amplitudes \\
PIT-MSE & $C^{-1}(\mathbb{E}[C(Y)\mid x])$ & Bounded but biased \\
\textbf{PIT-SUN} & $m(x)$ & Bounded coordinate + recovery \\
\bottomrule
\end{tabular}
}
\caption{Relationship between original MSE and PIT-SUN.}
\label{tab:mse_relation}
\end{table}

\subsection{B.2 Is PIT-SUN more than replacing TranSUN/GTS's manual transform with PIT?}
Yes. PIT-SUN inherits the conditional-linearity insight of TranSUN/GTS, but the algorithmic object being selected is different. The first distinction is structural: by the affine-inverse argument in Appendix A, no nonlinear stabilizing coordinate can be directly inverted into a universally valid conditional mean. Therefore the question is not whether PIT is a better manual transform, but how to close a nonlinear coordinate with an expectation-preserving recovery path.

The second distinction is the base. A direct plug-in replacement, PIT-TranSUN, sets $T(y)=\Phi^{-1}(\hat F_Y(y))$ in the TranSUN/GTS template and uses a bare inverse lookup as the base. This baseline benefits from empirical PIT, but remains fragile near zero plateaus and sparse tails because denominator collapse amplifies the proportional recovery labels. PIT-SUN instead treats the empirical marginal table as a closure object: the same table defines the bounded PIT coordinate, its generalized inverse, the stabilized inverse-quantile base, deployment lookup, and drift monitoring. PIT-SUN-ZI is a sparse-value extension of the same closure: it separates zero occurrence from positive amount while keeping the amount branch expectation-consistent through SUN recovery.

This is why the contribution is not a plug-in transform. PIT alone solves bounded representation but not expectation recovery; TranSUN/GTS solve conditional linearity but leave transform and base manual; PIT-TranSUN tests their naive composition and exposes base instability. PIT-SUN closes all three gaps through one empirical marginal object.

\subsection{B.3 How does PIT-SUN improve the constants behind recovery?}
Recovery methods such as TranSUN/GTS rely on conditionally linear recovery, but their finite-sample behavior still depends on the manually chosen transform $T$ and recovery base. PIT-SUN keeps the same recovery principle while replacing the uncertain transform-dependent constant with the controlled PIT constant $a_\delta$.

\begin{table}[h]
\centering
{\fontsize{9pt}{10pt}\selectfont\setlength{\tabcolsep}{3pt}
\begin{tabular}{lcc}
\toprule
Source & Manual transform & PIT-SUN \\
\midrule
Label scale & $B_T$ depends on tail & $a_\delta$ is fixed \\
Gradient bound & $2(a_f+B_T)G_f$ & $2(a_f+a_\delta)G_f$ \\
Cond. noise & No uniform bound & $\le a_\delta^2$ \\
Rec. variance & Base mismatch risk & Clipped by $b_{\min}$ \\
\bottomrule
\end{tabular}
}
\caption{Comparison of finite-sample constants.}
\label{tab:bounds}
\end{table}

The key point is that PIT-SUN explicitly controls the constants most responsible for mean collapse and transform sensitivity. The bounded PIT coordinate caps the main-branch label scale, and the lower-bounded inverse-quantile base caps recovery-label variance. When training data are limited, empirical CDF and tail-quantile estimates can be regularized or pooled, while the recovery architecture and its controlled constants stay unchanged.

\subsection{B.4 Why is $b_{\min}$ a variance clipper rather than a heuristic transform?}
$b_{\min}$ does not define the representation geometry. The representation is already determined by the empirical PIT coordinate $C(y)$. Instead, $b_{\min}$ only acts on the inverse-quantile recovery base in Eq.~\eqref{eq:base}. It is triggered when $C^{-1}(\bar{f}(x))$ falls into a near-zero plateau, which is common under zero-inflation.

Without this lower bound, the proportional target $Y/b(x)$ can have extremely large variance when $b(x)\approx\varepsilon$. Enforcing $b_{\min}$, such as a positive-target quantile in a small candidate range, regularizes the denominator while preserving the conditional linearity of the recovery subproblem. Therefore, $b_{\min}$ is a recovery-label variance control, not another manually selected target transform. The ablations and sensitivity tables report ratio variance, high-percentile ratios, and clipped fraction precisely to make this dependence visible rather than hiding it as automatic tuning.

\subsection{B.5 What theoretical role does each PIT-SUN component play?}
The analysis decomposes PIT-SUN into three linked operators. First, empirical PIT creates a bounded normal-score coordinate, replacing transform-dependent tail constants with the controlled clipping constant $a_\delta$. Second, the PIT-only bias result proves that this coordinate cannot be directly inverted when the target is the original-space expectation. Third, SUN recovery provides an expectation-preserving multiplicative operator under a fixed positive base, while the inverse-quantile base and $b_{\min}$ determine its finite-sample conditioning.

The recovery label has conditional variance $\operatorname{Var}(Y/b(x)\mid x)=\operatorname{Var}(Y\mid x)/b(x)^2$, so small bases amplify recovery noise. The inverse-quantile base acts as a data-adaptive scale proxy tied to the predicted PIT location, while $b_{\min}$ prevents zero-plateau denominators from exploding the ratio target. The empirical CDF error analysis then shows how quantile-table error propagates through the base and the final multiplicative prediction. Together, these results justify PIT-SUN as a coupled coordinate--base--recovery construction rather than a standalone transform or post-hoc retransformation correction.

The theory is therefore constructive: PIT supplies bounded main-branch labels, the PIT-only bias result rules out direct inverse prediction, SUN restores the expectation through a conditionally linear recovery path, and the inverse-quantile base with a lower floor keeps that path well-conditioned under sparse and heavy-tailed marginals. Model misspecification and joint optimization are evaluated through ablations, stop-gradient comparisons, sensitivity analyses, and deployment diagnostics.

\subsection{B.6 How does PIT-SUN handle empirical CDF errors and data drift?}
Extreme distribution tails are sparse by nature, and online recommender traffic may drift during promotions, holidays, or seasonal changes. Empirical quantiles at the top 0.1\% can fluctuate, and the derivative of $\Phi^{-1}$ increases near the boundary. PIT-SUN controls this source of error through four mechanisms.

First, $\delta$-clipping caps the normal-score range at $[-a_\delta,a_\delta]$, preventing boundary divergence. Second, industrial deployment uses large sliding-window quantile tables rather than noisy minibatch CDFs, reducing lookup variance. Third, the table is refreshed on recent traffic and monitored with quantile-drift statistics; when drift exceeds a validation-calibrated band, deployment refreshes the CDF, shortens the refresh interval, or upweights recent samples. In the anonymized diagnostic of Appendix D.5, a quantile-KS level around $0.04$ remains close to oracle-table behavior, while a stale-table level around $0.12$ coincides with degraded worst-block TRE; these are not universal constants, only an example of how the local error bound becomes an operational monitor. Fourth, PIT is not directly inverted as the final prediction; the SUN branch performs expectation recovery, so isolated tail lookup errors affect the scale base but do not by themselves define the final expectation estimate. In short, PIT-SUN does not assume a permanently stable empirical CDF; it maintains coordinate consistency through drift monitoring and sliding refresh. Appendix A gives a local error-propagation bound, and Appendix D includes a chronological drift diagnostic.

\subsection{B.7 Should tied values use deterministic mid-rank or randomized PIT?}
Both are valid, but they serve different data regimes. For massive zero-repeated targets such as dwell time and GMV, zero often has deterministic business semantics. Mid-rank PIT maps the zero mass into a stable quantile interval and avoids injecting unnecessary noise. In our default implementation, mid-rank means $\rho=0.5$ in Eq. (12) for both training and inference.

Randomized PIT is useful when repeated values arise from coarse quantization rather than a meaningful zero process. In that case, unfolding samples within a bin can reduce artificial ties and improve intra-bin ranking. When randomized PIT is enabled, $\rho$ is resampled only for tied training labels at the beginning of an epoch or cached with the training example for deterministic reproduction; validation, test, and online serving use deterministic mid-rank lookup unless the experiment explicitly reports the randomized variant. The rand-rank row in Table~\ref{tab:dwell} reports its effect on NMAE/TRE/xAUC and ratio behavior in a high-tie setting. Because clipping and atom handling replace the ideal smooth PIT with a generalized-inverse approximation, the Taylor statements in the theory should be read as interior smooth approximations plus boundary/atom remainders controlled by clipped mass and tie mass. In our main industrial setting, mid-rank PIT is preferred for semantic stability, while randomized PIT remains optional.

\subsection{B.8 When is PIT-SUN-ZI necessary?}
Single-track PIT-SUN handles zero mass through mid-rank PIT and $b_{\min}$. PIT-SUN-ZI further factorizes occurrence and positive amount when the deployment target requires explicit zero-positive monitoring. Practitioners should choose PIT-SUN-ZI over PIT-SUN when (i) Zero-AUC, positive recall, or conversion-boundary monitoring is a primary deployment metric, (ii) downstream policies consume probability and positive amount separately, or (iii) product debugging requires identifying whether errors come from occurrence prediction or positive-value recovery. Zero-AUC is therefore an important trigger, but not the only one: operational interpretability and policy modularity are equally relevant. If the main objective is end-to-end value calibration, TRE/NMAE, or ranking of expected value, the simpler single-track PIT-SUN remains the default estimator.

For example, in transaction modeling, business logic may require separately tracking conversion probability and positive transaction amount. PIT-SUN-ZI uses a hurdle track for $P(Y>0\mid x)$ and applies PIT-SUN only on $Y>0$ to estimate $\mathbb{E}[Y\mid x,Y>0]$. The final prediction is
\begin{equation}
\hat{y}(x)=P(Y>0\mid x)\,\mathbb{E}[Y\mid x,Y>0].
\end{equation}
This extension improves interpretability and modularity, while the default single-track PIT-SUN remains the simpler deployment choice for end-to-end value calibration.

\subsection{B.9 Why is stop-gradient used in joint optimization?}
The expectation-consistency theorem is most direct under a fixed-base recovery subproblem. In joint neural training, however, the recovery loss could otherwise update the shared representation and the inverse-quantile path through high-variance proportional labels. The stop-gradient operator preserves the fixed-base interpretation by forming $r=Y/\operatorname{sg}[b(x)]$: gradients from $\mathcal{L}_{\mathrm{SUN}}$ train $z_\psi$ and any shared layers through the ratio-head path, but they do not flow through the denominator lookup or floor.

This statement is architectural rather than a generic convergence claim: it enforces the branch separation required by PIT-SUN. PIT learns the stable coordinate, SUN recovery learns the expectation ratio, and stop-gradient prevents recovery noise from rewriting the coordinate through the inverse-quantile lookup. Two-stage training matches the fixed-base analysis most directly, while joint training with stop-gradient preserves the same separation in a simpler deployable pipeline.

\subsection{B.10 Can PIT-SUN be used as a multi-task scale-aligner?}
Yes. In multi-task learning systems such as MMoE, binary targets such as CTR are often trained together with continuous targets such as dwell time, GMV, or LTV. These continuous targets can differ by orders of magnitude, producing gradient conflicts and unstable loss weighting.

The empirical PIT branch maps each continuous target into a comparable normal-score range $[-a_\delta,a_\delta]$. This makes PIT-SUN a model-agnostic scale-aligner for heterogeneous regression tasks. The SUN branch then restores each task's original-space expectation, so scale alignment does not sacrifice business-scale calibration.

\subsection{B.11 What are the main limitations and boundary cases?}
PIT-SUN is designed for the regimes where heavy tails, sparsity, or heterogeneous target scales make original-space MSE difficult to optimize. When targets are already well-conditioned or sample sizes are too small for reliable tail quantiles, the empirical table can be pooled, smoothed, or refreshed over larger windows while preserving the same recovery architecture.

The method uses a global marginal CDF $F_Y$ by default because this choice gives one stable deployment object for coordinate construction, inverse lookup, recovery-base construction, and drift monitoring. Under strong subgroup heterogeneity, smoothed, pooled, or segment-aware CDF tables can be introduced when enough per-segment data are available, as evaluated in Appendix D. A practical rule is to keep the global table unless subgroup quantile KS and bucket-level calibration gaps persist beyond the validation-calibrated drift band; segment-aware tables are used only when each segment has enough recent samples to estimate stable tail quantiles.

Finally, the CDF table should be built only from training labels or causally available recent traffic. Offline test labels are never used for lookup construction, and online tables should be refreshed and monitored for drift.

\section{Appendix C: Experimental Settings}
\label{app:exp_setting}

Appendix C documents the experimental settings in the same order as the main evaluation pipeline: synthetic experiments, public benchmarks, and industrial datasets. The synthetic setting includes the full generation protocol and reference sketch, while Appendix D provides additional experimental supplements and analyses.

\subsection{C.1 Synthetic Experiment Setting}
Figure~\ref{fig:syn_dist} visualizes the 12 synthetic datasets used for the synthetic robustness study. The suite deliberately spans right-skewed heavy-tailed, left-skewed, and symmetric target marginals induced by different conditional noise families, so that robustness can be tested beyond a single analytical family. The supervised data-generation mechanism follows the standard order $X \rightarrow m(X) \rightarrow Y\mid X$; the detailed construction and reference implementation are included below as part of this synthetic setting.

\begin{figure*}[h]
    \centering
    \includegraphics[width=0.85\textwidth]
    {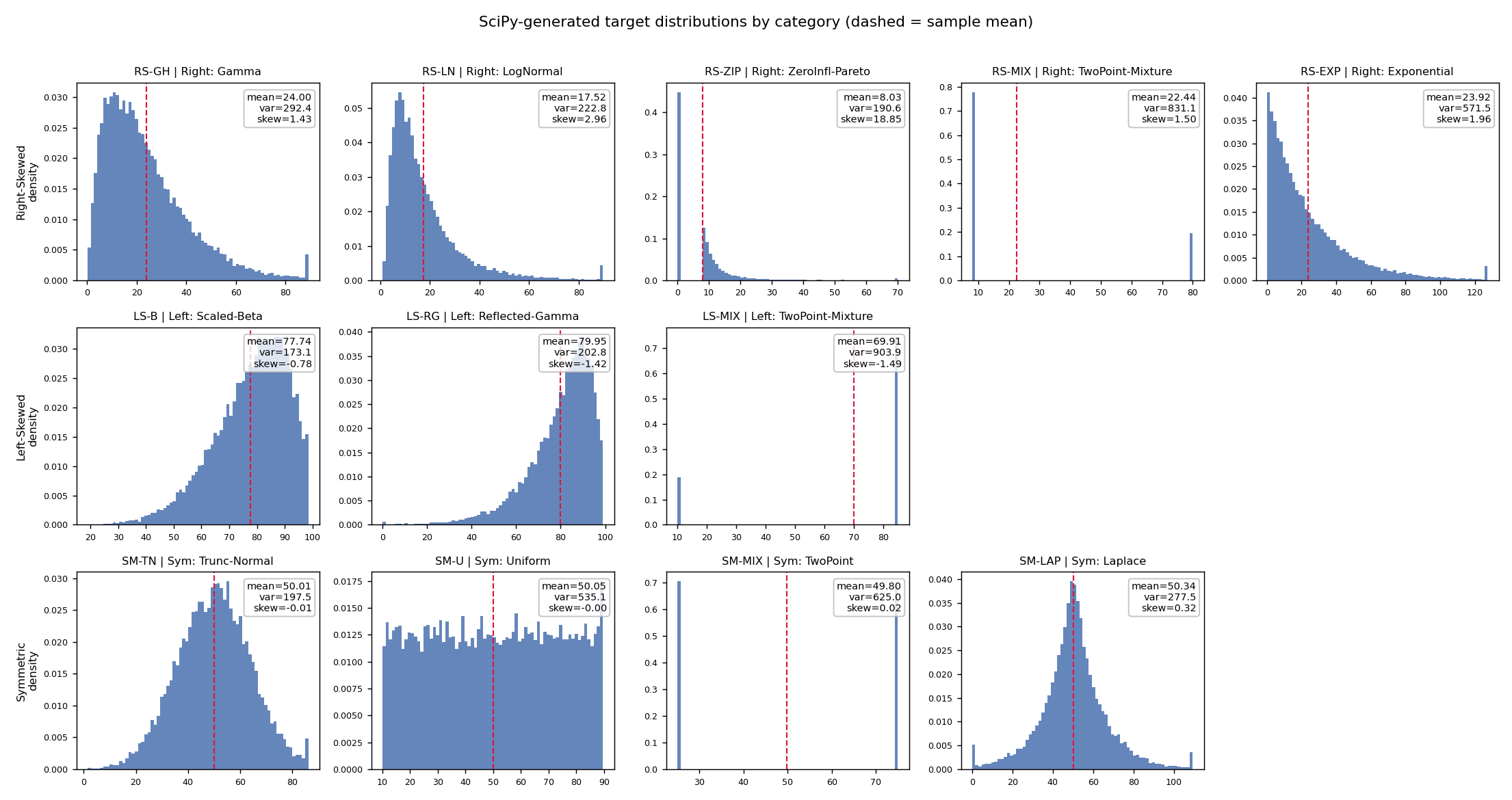}
    \caption{Marginal distributions of the 12 synthetic datasets used for robustness evaluation under heterogeneous empirical marginals.}
    \label{fig:syn_dist}
\end{figure*}

\begin{table}[h]
\centering
{\fontsize{9pt}{10pt}\selectfont\setlength{\tabcolsep}{3pt}
\begin{tabular}{llrr}
\toprule
Abbr & Category: Name & Mean($Y$) & Skew($Y$) \\
\midrule
RS-GH & right: Gamma & 23.996 & 1.431 \\
RS-LN & right: LogNormal & 17.521 & 2.964 \\
RS-ZIP & right: ZeroInfl-Pareto & 8.028 & 18.854 \\
RS-MIX & right: TwoPoint-Mixture & 22.440 & 1.496 \\
RS-EXP & right: Exponential & 23.918 & 1.964 \\
LS-B & left: Scaled-Beta & 77.737 & -0.782 \\
LS-RG & left: Reflected-Gamma & 79.953 & -1.420 \\
LS-MIX & left: TwoPoint-Mixture & 69.914 & -1.491 \\
SM-TN & sym: Trunc-Normal & 50.008 & -0.013 \\
SM-U & sym: Uniform & 50.046 & -0.000 \\
SM-MIX & sym: TwoPoint & 49.795 & 0.016 \\
SM-LAP & sym: Laplace & 50.345 & 0.321 \\
\bottomrule
\end{tabular}
}
\caption{Synthetic dataset statistics. The suite covers right-skewed, left-skewed, and near-symmetric marginals.}
\label{tab:synthetic_stats}
\end{table}

This subsection provides the detailed synthetic setting specification used for the robustness experiments in the main text. It documents the supervised generator as part of Appendix C's experimental settings. The goal is to evaluate expectation recovery under controlled supervised regression problems, while varying the marginal tail shape and zero inflation without leaking target ranks into the features.

\subsubsection{Synthetic Data Generation Protocol}
\label{app:synthetic_protocol}
\paragraph{Construction Mechanics.}
Each dataset first samples covariates $X\sim\mathcal{N}(0,I_d)$ and defines a known conditional mean $m(X)$. We use either a smooth single-index mean $m(X)=\operatorname{softplus}(w^\top X+c)$ or a lightweight mixture-of-experts mean whose gate is a logistic function of $X$. The observed target is then sampled from a conditional distribution $Y\mid X$ whose expectation is $m(X)$ but whose noise family controls the marginal difficulty. For example, lognormal and Gamma variants are parameterized so their conditional means equal $m(X)$; zero-inflated variants sample $Y=0$ with probability $\pi(X)$ and otherwise sample a positive amount with mean $m(X)/(1-\pi(X))$; mixture variants combine low- and high-scale components with an $X$-dependent gate.

This construction avoids using the empirical rank of $Y$ as an input feature. The features are sampled before the target, the ground-truth $m(X)=\mathbb{E}[Y\mid X]$ is known analytically for the main variants or estimated by high-sample Monte Carlo for mixture variants, and evaluation compares original-space expectation error rather than reconstruction from rank-coded features.

The 12 variants cover three regimes: five right-skewed heavy-tailed conditional families (Gamma, LogNormal, zero-inflated Pareto-like positive amounts, two-component right mixture, and Exponential), three left-skewed bounded or reflected families (Scaled-Beta, Reflected-Gamma, and two-component left mixture), and four symmetric or near-symmetric families (Truncated-Normal, Uniform-noise, two-component symmetric mixture, and Laplace). All variants use $N_{\mathrm{train}}=80{,}000$ and $N_{\mathrm{test}}=20{,}000$. The oracle reference for synthetic evaluation is the known conditional expectation $m(X)$, so the benchmark isolates original-space expectation estimation under controlled marginal and conditional-noise shapes.

\paragraph{Reference Implementation Sketch.}
Listing~\ref{lst:synthetic_protocol} summarizes the generation logic in a compact, implementation-agnostic form. The key idea is to vary the conditional noise family while keeping covariate distribution, mean function, sample size, model capacity, and evaluation protocol fixed across all synthetic datasets.

\begin{center}
\captionof{listing}{Reference sketch for the synthetic data generation protocol.}
\label{lst:synthetic_protocol}
\end{center}
\begin{lstlisting}[language=Python]
N_train, N_test, d = 80000, 20000, 8

families = {
    # conditional families with E[Y|X] = m(X)
    "RS-GH":  GammaNoise(shape=2.0),
    "RS-LN":  LogNormalNoise(sigma=0.75),
    "RS-ZIP": ZeroInflated(ParetoLikePositive(shape=2.2), pi_x),
    "RS-MIX": RightMixture(low_scale=0.6, high_scale=3.0, gate_x),
    "RS-EXP": ExponentialNoise(),

    # left-skewed / bounded variants
    "LS-B":   ScaledBetaNoise(alpha=7.0, beta=2.0),
    "LS-RG":  ReflectedGammaNoise(shape=2.0),
    "LS-MIX": LeftMixture(low_scale=0.4, high_scale=1.2, gate_x),

    # symmetric or near-symmetric variants
    "SM-TN":  TruncatedNormalNoise(std=0.35),
    "SM-U":   UniformNoise(width=0.5),
    "SM-MIX": SymmetricMixture(scale_a=0.7, scale_b=1.3),
    "SM-LAP": LaplaceNoise(scale=0.4),
}

for name, noise_family in families.items():
    X = sample_normal_features(N_train + N_test, d)
    m = softplus(X @ w + c)          # or an MoE mean m(X)
    y = sample_conditional_target(noise_family, mean=m)
    train, test = chronological_or_random_split((X, y, m), sizes=[N_train, N_test])
    evaluate_all_methods(train, test, target="E[Y|X]", oracle_mean=m)
\end{lstlisting}

\paragraph{Reproducibility Notes.}
All synthetic results use the same training sizes, feature dimension, random-seed grid, optimizer settings, PIT clipping probability, quantile-table capacity, and $b_{\min}$ rule reported in Appendix C. The conditional noise family is the controlled factor varied across the 12 synthetic datasets, enabling the SRE table to isolate empirical-marginal recovery robustness rather than target-rank leakage, model-capacity, or data-volume effects.

\subsection{C.2 Public Benchmark Setting}
We utilize two public benchmarks: CIKM16 and DTMart. \textbf{CIKM16} predicts the dwell time of online searching sessions using intra-session item features, comprising 310,302 sessions and 122,991 items. \textbf{DTMart} forecasts the transaction amount per order in marketing activities, incorporating item IDs, categories, procurement, and delivery features across 985,033 orders and 15,546 items. Both are partitioned using a 4:1 random train/test split.

\subsection{C.3 Industrial Dataset Setting}
The two large-scale industrial datasets (\textbf{Dwell Time} and \textbf{GMV}) adhere to a strict chronological split protocol. Both are anonymized at the tens-of-millions scale for training and testing, with over 1,000 dense and sparse features. The split is constructed by time so that training windows precede validation and test windows; CDF/quantile tables are built only from training labels or validation-selected training windows. Business-control features that are only available after exposure, allocation, or downstream feedback are excluded from the offline feature set, and the same feature availability mask is applied to every baseline. We additionally perform a leakage audit by checking timestamp monotonicity, removing post-treatment aggregates, verifying that lookup tables never consume validation/test labels, and confirming that model inputs are identical across treatment and baseline variants except for method-specific heads or losses. For the industrial comparisons in Tables~\ref{tab:dwell}--\ref{tab:gmv}, all baselines and PIT-SUN variants are implemented in the same production training stack with the same MLP backbone architecture, identical dense/sparse feature set, common optimizer family and schedule, and the same preprocessing pipeline. Industrial metrics are normalized from the same held-out labels using identical denominators, pair sampling, and aggregation windows for every model. The shared backbone is an MLP optimized over 5 epochs. Table~\ref{tab:desc_stat} reports compliance-preserving descriptive statistics. We avoid disclosing overly granular quantiles or exact sample counts; instead, the table highlights distributional shape, sparsity, concentration, and tail severity through normalized or ratio-based indicators. These statistics validate the marginal complexity discussed in the main text and justify evaluating point errors, aggregate calibration, and ranking metrics jointly.

\begin{table}[h]
\centering
{\fontsize{9pt}{10pt}\selectfont\setlength{\tabcolsep}{3pt}
\begin{tabular}{lcc}
\toprule
Statistic & Dwell Time & GMV \\
\midrule
Scale & 100+ millions & 100+ millions \\
Feature Dim. & $>1{,}000$ & $>1{,}000$ \\
Zero Rate & 0.150 & 0.299 \\
Coeff. of Variation & 1.56 & 5.40 \\
Skewness & 2.75 & 60.58 \\
Excess Kurtosis & 19.71 & $1.08\times10^{4}$ \\
Top 1\% Share & 0.088 & 0.339 \\
Top 5\% Share & 0.300 & 0.628 \\
\bottomrule
\end{tabular}
}
\caption{Anonymized descriptive statistics of large-scale industrial datasets. Statistics are compliance-preserving summaries of shape, concentration, and tail severity.}
\label{tab:desc_stat}
\end{table}

\subsection{C.4 Shared Training Protocols and Evaluation Metrics}
Table~\ref{tab:train_proto} summarizes the primary offline training protocols uniformly applied across the synthetic, public, and industrial settings. The offline comparisons in Tables~\ref{tab:public}--\ref{tab:gmv} share the same backbone within each dataset group, the same feature set and split, common optimizer settings whenever applicable, and the same validation protocol for method-specific choices such as transform type, PIT clipping, table size, $b_{\min}$, and recovery weight. No method receives additional tuning after inspecting the test set.

\begin{table}[h]
\centering
{\fontsize{8pt}{9pt}\selectfont\setlength{\tabcolsep}{2pt}
\begin{tabular}{lccc}
\toprule
Setting & Synthetic & CIKM16 / DTMart & Dwell/GMV \\
\midrule
Split & 80k / 20k & Random 4:1 & Time split \\
Feature dim. & 8 & Dataset-specific & $>1{,}000$ \\
Seeds & 5 & 5 & 3 \\
Optimizer & Adam & Adam & AdamW \\
Learning rate & 1e-3 & 1e-3 & 5e-4 \\
Batch size & 1,024 & 2,048 & 8,192 \\
Epochs & 80 & 30 & 5 \\
Quantile table $K$ & 4,000 & 8,000 & 20,000 \\
PIT clip $\delta$ & $10^{-4}$ & $10^{-4}$ & $10^{-4}$ \\
$b_{\min}$ & Pos. $P_{10}$ & Pos. $P_{10}$ & Pos. $P_{10}$ \\
$\lambda$ & 1.0 & 1.0 & 1.0 \\
\bottomrule
\end{tabular}
}
\caption{Primary offline training protocols across datasets.}
\label{tab:train_proto}
\end{table}

\paragraph{Industrial online A/B isolation.} The online experiment uses randomized user buckets in the production allocation framework. The control is the existing mixed production strategy, while treatment replaces only the value-estimation component with PIT-SUN; allocator, exploration policy, ranking/serving modules, budget constraints, guardrails, and feature eligibility remain fixed. The online setup uses production features, continuous training/serving schedules, sliding quantile tables, and 7 daily blocks for stability inspection. The single-objective setting uses dwell-time reward, and the multi-objective setting combines normalized dwell-time, engagement-depth, and commercial-value rewards under the same reward combiner and pacing rules. Reported lifts use user-level bootstrap confidence intervals; exact user counts and traffic percentages remain anonymized for compliance.

We systematically evaluate the regression estimators across three critical dimensions.
\textbf{Point Estimation Error:} NMAE and NRMSE are normalized by the target mean $\bar{y}=\frac{1}{N}\sum_i y_i$ to facilitate cross-domain comparisons:
\begin{align}
\mathrm{NMAE}
&=\frac{\frac{1}{N}\sum_i |\hat{y}_i-y_i|}{\bar{y}+\epsilon},\\
\mathrm{NRMSE}
&=\frac{\sqrt{\frac{1}{N}\sum_i(\hat{y}_i-y_i)^2}}{\bar{y}+\epsilon}.
\end{align}
\textbf{Aggregate Business Calibration:} The Signal Relative Error (SRE) quantifies the absolute relative deviation of the aggregated predicted signal versus the true signal:
\begin{equation}
\mathrm{SRE} = \frac{\left|\sum_i \hat{y}_i - \sum_i y_i\right|}{\sum_i y_i + \epsilon}.
\end{equation}
We additionally report TRE, MRE, and PGR:
\begin{equation}
\mathrm{TRE}=\frac{\sum_i |\hat{y}_i-y_i|}{\sum_i y_i+\epsilon},
\quad
\mathrm{MRE}=\frac{1}{N}\sum_i\frac{|\hat{y}_i-y_i|}{y_i+\epsilon},
\end{equation}
\begin{equation}
\mathrm{PGR}=\frac{\frac{1}{N}\sum_i \hat{y}_i-\frac{1}{N}\sum_i y_i}{\frac{1}{N}\sum_i y_i+\epsilon}.
\end{equation}
Here TRE measures total residual mass relative to true mass, MRE measures instance-wise relative error, and PGR preserves the sign of aggregate prediction gain to diagnose overestimation or underestimation.
\textbf{Ranking Quality:} xAUC evaluates pairwise ranking correctness for high-value continuous targets. We sample video pairs $(i,j)$ from the test set; if the predicted dwell-time order and the true dwell-time order agree, the pair receives score 1, otherwise 0, following an AUC-like pairwise comparison:
\begin{equation}
\mathrm{xAUC}=\mathbb{E}_{(i,j)}\left[\mathbb{I}\left((\hat{y}_i-\hat{y}_j)(y_i-y_j)>0\right)\right].
\end{equation}
Normalized Gini assesses the capability to capture true cumulative values upon prediction-based sorting. Spearman's $\rho$ measures global monotonic consistency.

\begin{table*}[t]
\centering
{\fontsize{9pt}{10pt}\selectfont\setlength{\tabcolsep}{3pt}
\begin{tabular}{lcccccccccccc}
\toprule
Model & RS-GH & RS-LN & RS-ZIP & RS-MIX & RS-EXP & LS-B & LS-RG & LS-MIX & SM-TN & SM-U & SM-MIX & SM-LAP \\
\midrule
MSE & \underline{0.0473} & \underline{0.0090} & 0.0518 & 0.0589 & 0.0341 & 0.0087 & 0.0235 & 0.0126 & 0.0512 & \textbf{0.0064} & \textbf{0.0085} & 0.0334 \\
T-MSE(ln) & 0.1481 & 0.3166 & 0.4874 & 0.5965 & 0.3496 & 0.0075 & 0.0161 & 0.0235 & 0.0441 & 0.0878 & 0.3009 & 0.0964 \\
T-MSE(sqrt) & 0.0609 & 0.2144 & 0.4594 & 0.3709 & 0.2365 & 0.0069 & 0.0126 & 0.1473 & 0.0078 & 0.0823 & 0.1143 & 0.0478 \\
T-MSE($y^2$) & 0.2168 & 0.7787 & 0.8722 & 0.6912 & 0.4732 & 0.0541 & 0.0348 & 0.1099 & 0.0542 & 0.0580 & 0.1385 & 0.0663 \\
TranSUN(ln) & 0.0558 & 0.0904 & \textbf{0.0502} & \underline{0.0505} & 0.0828 & \underline{0.0052} & 0.0230 & 0.0194 & 0.0269 & 0.0151 & 0.0579 & 0.0350 \\
TranSUN(sqrt) & 0.0558 & 0.0517 & 0.0723 & 0.0520 & 0.0479 & \textbf{0.0042} & 0.0196 & 0.0125 & 0.0248 & 0.0154 & 0.0503 & 0.0194 \\
TranSUN($y^2$) & 0.0716 & 0.5223 & 0.0509 & 0.0704 & \textbf{0.0133} & 0.0114 & 0.0296 & \underline{0.0112} & 0.0364 & 0.0231 & 0.0587 & 0.0173 \\
GTS(ln) & 0.0783 & 0.1899 & 0.1521 & 0.0537 & 0.0419 & 0.0091 & \textbf{0.0078} & 0.0405 & \textbf{0.0049} & 0.0176 & 0.0220 & \underline{0.0127} \\
GTS(sqrt) & 0.0717 & 0.1725 & 0.1433 & 0.0512 & 0.0304 & 0.0065 & 0.0103 & 0.0248 & 0.0086 & 0.0180 & 0.0290 & 0.0136 \\
\midrule
PIT-SUN (ours) & \textbf{0.0280} & \textbf{0.0031} & \underline{0.0506} & \textbf{0.0474} & \underline{0.0249} & 0.0053 & \underline{0.0081} & \textbf{0.0106} & \underline{0.0076} & \underline{0.0132} & \underline{0.0143} & \textbf{0.0107} \\
\bottomrule
\end{tabular}
}
\caption{SRE ($\downarrow$) on 12 synthetic distributions ($p<0.05$). Best results are in bold, second-best are underlined.}
\label{tab:synthetic_sre_app}
\end{table*}

\section{Appendix D: Experimental Result Analyses}
\label{app:supp_exp}

Appendix D provides supplementary analyses. It covers full synthetic SRE results, fixed/oracle transform stability, public multi-seed stability, and industrial diagnostics on sensitivity, drift, bucket behavior, zero-positive decoupling, and optimization.

\subsection{D.1 Synthetic Full Results}
Table~\ref{tab:synthetic_sre_app} reports the complete SRE values behind Figure~\ref{fig:synthetic_radar}. The main text uses a radar chart to emphasize robustness patterns, while the full table is given here. Together with the five-seed protocol in Appendix C.4, the public stability table, and the fixed/oracle summary, these results show that the gains are not driven by a single seed; the radar chart is only a qualitative visualization, not a substitute for the numerical tables.

\subsection{D.2 Synthetic Fixed/Oracle Transform Stability}
Table~\ref{tab:multi_seed_oracle} summarizes the synthetic fixed/oracle transform comparison. PIT-SUN obtains the lowest average SRE without per-dataset transform selection, and its average and worst-case ranks remain close to the oracle recovery transform that is allowed to choose the best recovery transform after observing each dataset.

\begin{center}
{\fontsize{9pt}{10pt}\selectfont\setlength{\tabcolsep}{3pt}
\begin{tabular}{lrrrr}
\toprule
Selection & Avg. SRE $\downarrow$ & Avg. Rank $\downarrow$ & Worst $\downarrow$ & Wins \\
\midrule
Fixed log & 0.2062 & 8.25 & 10 & 0 \\
Fixed sqrt & 0.1468 & 6.92 & 10 & 0 \\
Fixed $y^2$ & 0.2957 & 9.58 & 10 & 0 \\
Oracle fixed & 0.1344 & 6.50 & 8 & 0 \\
Oracle rec. & 0.0249 & \textbf{1.83} & \textbf{3} & \textbf{6} \\
\midrule
\textbf{PIT-SUN} & \textbf{0.0191} & 1.92 & 4 & 4 \\
\bottomrule
\end{tabular}
}
\captionof{table}{Synthetic fixed/oracle transform stability comparison. Lower average SRE, average rank, and worst rank indicate stronger stability; wins count per-dataset first places.}
\label{tab:multi_seed_oracle}
\end{center}
The comparison shows that fixed analytical transforms do not cover heterogeneous marginals, and even oracle fixed-transform selection remains well behind recovery-based alternatives. PIT-SUN approaches the oracle recovery-transform ranking without manual transform choice or per-dataset posterior selection.

\subsection{D.3 Public Benchmark Stability}
Table~\ref{tab:public_multi} reports multi-seed mean and standard deviation on the two public benchmarks for PIT-SUN and the strongest/second-best competitors from Table~\ref{tab:public}. The table is an uncertainty audit for decisive comparisons, not a full repetition of every baseline row.

\begin{center}
{\fontsize{9pt}{10pt}\selectfont\setlength{\tabcolsep}{1.2pt}
\begin{tabular}{lcccc}
\toprule
\multirow{2}{*}{Model} & \multicolumn{2}{c}{CIKM16} & \multicolumn{2}{c}{DTMart} \\
\cmidrule(lr){2-3} \cmidrule(lr){4-5}
 & NMAE $\downarrow$ & xAUC $\uparrow$ & NMAE $\downarrow$ & xAUC $\uparrow$ \\
\midrule
TPM & $0.4410_{\pm.002}$ & $0.6831_{\pm.001}$ & $0.2230_{\pm.001}$ & $0.9350_{\pm.001}$ \\
TranSUN & $0.4370_{\pm.001}$ & $0.6782_{\pm.001}$ & $0.2246_{\pm.001}$ & $0.9359_{\pm.001}$ \\
GR & $0.4820_{\pm.002}$ & $0.6880_{\pm.001}$ & $0.2650_{\pm.002}$ & $0.9270_{\pm.001}$ \\
CCOR-Net& $0.4492_{\pm.002}$ & $0.6795_{\pm.001}$ & $0.2746_{\pm.002}$ & $0.9304_{\pm.001}$ \\
PIT-TranSUN & $0.4470_{\pm.001}$ & $0.6820_{\pm.001}$ & $0.2380_{\pm.001}$ & $0.9370_{\pm.001}$ \\
\textbf{PIT-SUN} & \textbf{0.4240$_{\pm.001}$} & \textbf{0.6952$_{\pm.001}$} & \textbf{0.2070$_{\pm.001}$} & \textbf{0.9430$_{\pm.000}$} \\
\bottomrule
\end{tabular}
}
\captionof{table}{Multi-seed Mean $\pm$ Std on Public Datasets, using the same metric definitions as Table~\ref{tab:public}.}
\label{tab:public_multi}
\end{center}

\subsection{D.4 Industrial Hyperparameter Sensitivity}
Tables~\ref{tab:delta_sens} to \ref{tab:lambda_sens} present sensitivity analyses for the clipping probability $\delta$, quantile table capacity $K$, lower-bound quantile $q_b$, and loss weight $\lambda$. We select $q_b$ on validation from $\{P_1,P_5,P_{10},P_{20},P_{30}\}$ by jointly checking NMAE/TRE, Ratio Var, tail-ratio percentiles, and floor-active fraction. The trends support the core claims: PIT-SUN remains stable across broad operational ranges, while omitting $b_{\min}$ or aggressively truncating the tail causes severe calibration degradation.

\begin{table}[h]
\centering
{\fontsize{9pt}{10pt}\selectfont\setlength{\tabcolsep}{2pt}
\begin{tabular}{lcccccc}
\toprule
$\delta$ & NMAE & NRMSE & TRE & MRE & xAUC & Tail \\
\midrule
$10^{-3}$ & 0.486 & 0.904 & 0.058 & 0.846 & 0.918 & 0.083 \\
$10^{-4}$ & \textbf{0.477} & \textbf{0.885} & \textbf{0.051} & \textbf{0.818} & \textbf{0.921} & \textbf{0.071} \\
$10^{-5}$ & 0.481 & 0.892 & 0.054 & 0.827 & 0.920 & 0.075 \\
$N^{-1/3}$ & 0.479 & 0.888 & 0.052 & 0.821 & 0.918 & 0.073 \\
$N^{-1/4}$ & 0.493 & 0.913 & 0.064 & 0.858 & 0.916 & 0.091 \\
Heuristic & 0.489 & 0.907 & 0.061 & 0.861 & 0.919 & 0.087 \\
\bottomrule
\end{tabular}
}
\caption{Sensitivity of PIT clipping probability $\delta$ (Dwell Time). Lower is better except xAUC.}
\label{tab:delta_sens}
\end{table}

\begin{table}[h]
\centering
{\fontsize{9pt}{10pt}\selectfont\setlength{\tabcolsep}{3pt}
\begin{tabular}{rccccc}
\toprule
$K$ & NMAE $\downarrow$ & NRMSE $\downarrow$ & TRE $\downarrow$ & xAUC $\uparrow$ & Inf. P50 \\
\midrule
1,000 & 0.489 & 0.913 & 0.061 & 0.916 & 14.37 ms \\
4,000 & 0.481 & 0.895 & 0.055 & 0.919 & 14.48 ms \\
8,000 & 0.479 & 0.889 & 0.053 & 0.920 & 14.55 ms \\
20,000 & \textbf{0.477} & \textbf{0.885} & \textbf{0.051} & \textbf{0.921} & 14.61 ms \\
50,000 & 0.482 & 0.892 & 0.054 & 0.916 & 14.88 ms \\
\bottomrule
\end{tabular}
}
\caption{Sensitivity of quantile table size $K$ (Dwell Time).}
\label{tab:K_sens}
\end{table}

Tables~\ref{tab:bmin_sens_dur} and \ref{tab:bmin_sens_gmv} show that $b_{\min}$ acts as a \textit{variance clipper} rather than a heuristic transform. Removing it causes ratio-variance explosions, while overly high floors trade variance reduction for base distortion through a larger floor-active fraction. The preferred $P_{10}$ balances these effects.

\begin{table}[h]
\centering
{\fontsize{9pt}{10pt}\selectfont\setlength{\tabcolsep}{2pt}
\begin{tabular}{lcccc}
\toprule
Metric & none & $P_1$ & $P_{10}$ & $P_{30}$ \\
\midrule
NMAE $\downarrow$ & 0.801 & 0.514 & \textbf{0.477} & 0.492 \\
NRMSE $\downarrow$ & 1.086 & 0.934 & \textbf{0.885} & 0.914 \\
TRE $\downarrow$ & 0.083 & 0.067 & \textbf{0.051} & 0.065 \\
xAUC $\uparrow$ & 0.910 & 0.916 & \textbf{0.921} & 0.917 \\
Ratio Var $\downarrow$ & 18.72 & 6.91 & 3.84 & \textbf{3.55} \\
Ratio P99 $\downarrow$ & 41.8 & 18.6 & 10.7 & \textbf{8.9} \\
Floor active & 0.000 & 0.006 & 0.057 & 0.178 \\
\bottomrule
\end{tabular}
}
\caption{Sensitivity of $b_{\min}$ quantile $q_b$ (Dwell Time).}
\label{tab:bmin_sens_dur}
\end{table}

\begin{table}[h]
\centering
{\fontsize{8pt}{9pt}\selectfont\setlength{\tabcolsep}{1pt}
\begin{tabular}{lcccccc}
\toprule
Metric & none & $P_1$ & $P_5$ & $P_{10}$ & $P_{20}$ & $P_{30}$ \\
\midrule
NMAE $\downarrow$ & 1.009 & 0.671 & 0.638 & \textbf{0.623} & 0.635 & 0.647 \\
NRMSE $\downarrow$ & 4.331 & 4.006 & 3.902 & \textbf{3.853} & 3.917 & 3.951 \\
TRE $\downarrow$ & 0.298 & 0.109 & 0.083 & \textbf{0.071} & 0.088 & 0.101 \\
MRE $\downarrow$ & 1.623 & 1.256 & 1.181 & \textbf{1.143} & 1.176 & 1.202 \\
xAUC $\uparrow$ & 0.872 & 0.902 & 0.907 & \textbf{0.909} & 0.906 & 0.903 \\
Ratio Var $\downarrow$ & 43.6 & 14.8 & 10.4 & 8.9 & 8.3 & \textbf{8.1} \\
Ratio P99.9 $\downarrow$ & 192.4 & 63.7 & 47.8 & 41.6 & 38.4 & \textbf{36.8} \\
Floor active & 0.000 & 0.011 & 0.046 & 0.092 & 0.161 & 0.238 \\
\bottomrule
\end{tabular}
}
\caption{Sensitivity of $b_{\min}$ quantile $q_b$ (GMV).}
\label{tab:bmin_sens_gmv}
\end{table}

\begin{center}
{\fontsize{9pt}{10pt}\selectfont\setlength{\tabcolsep}{3pt}
\begin{tabular}{lcccccc}
\toprule
$\lambda$ & NMAE $\downarrow$ & NRMSE $\downarrow$ & TRE $\downarrow$ & MRE $\downarrow$ & xAUC $\uparrow$ & Spear$\rho$ $\uparrow$ \\
\midrule
0.25 & 0.493 & 0.910 & 0.067 & 0.904 & 0.916 & 0.834 \\
1.00 & \textbf{0.477} & \textbf{0.885} & \textbf{0.051} & \textbf{0.818} & \textbf{0.921} & \textbf{0.852} \\
4.00 & 0.493 & 0.908 & 0.071 & 0.873 & 0.916 & 0.838 \\
\bottomrule
\end{tabular}
}
\captionof{table}{Sensitivity of recovery loss weight $\lambda$ (Dwell Time).}
\label{tab:lambda_sens}
\end{center}
The middle setting $\lambda=1.0$ best balances PIT geometry and the recovery objective. Smaller weights underuse expectation correction, whereas larger weights overemphasize ratio fitting and mildly hurt calibration and ranking.

\subsection{D.5 Industrial Chronological Drift Diagnostics}
\label{app:drift_diag}
Because PIT-SUN depends on an empirical CDF table, a key deployment concern is whether an outdated table distorts the PIT coordinate under traffic shifts. We therefore run a chronological diagnostic on industrial Dwell Time by splitting the test horizon into consecutive blocks and comparing three table policies: a stale table frozen from the earliest block, a sliding-window table refreshed on recent traffic, and an oracle table rebuilt from the current block for diagnosis only. The model is fixed and only the lookup table changes, isolating CDF drift. Operationally, quantile KS is computed between the active table and a recent shadow table on causally available traffic. A validation-calibrated warning band triggers early refresh or shorter intervals, while persistently high discrepancy near the stale-table regime triggers table rebuild and bucket-level calibration review.
\begin{center}
{\fontsize{9pt}{10pt}\selectfont\setlength{\tabcolsep}{4pt}
\begin{tabular}{lccc}
\toprule
Metric & Frozen & Sliding & Oracle \\
\midrule
Avg. TRE $\downarrow$ & 0.073 & \textbf{0.054} & 0.052 \\
Worst-block TRE $\downarrow$ & 0.104 & \textbf{0.069} & 0.066 \\
Avg. xAUC $\uparrow$ & 0.914 & \textbf{0.920} & 0.921 \\
Quantile KS $\downarrow$ & 0.118 & 0.041 & \textbf{0.000} \\
Refresh interval & stale & sliding & per block \\
Refresh triggers & -- & KS/calib. band & diagnostic \\
Triggered blocks & -- & low & all \\
\bottomrule
\end{tabular}
}
\captionof{table}{Chronological CDF-drift diagnostic on Dwell Time. Oracle is diagnostic only; deployment uses sliding-window refresh. Quantile KS measures lookup-table discrepancy. Refresh triggers are validation-calibrated operational rules rather than universal thresholds; ``low'' denotes a small number of early-refresh events under anonymized production reporting.}
\label{tab:drift_diag}
\end{center}

\begin{table*}[t]
\centering
{\fontsize{9pt}{10pt}\selectfont\setlength{\tabcolsep}{3pt}
\begin{tabular}{lllcccccc}
\toprule
Domain & True Target Bucket & Metric & GTS(sqrt) & CREAD & PIT-only & PIT-SUN & PIT-SUN-ZI & ZI w/o Amt-SUN \\
\midrule
\multirow{4}{*}{Dwell Time} & $[0, 50\%)$ & MAE $\downarrow$ & 0.184 & 0.176 & 0.171 & \textbf{0.166} & -- & -- \\
 & $[50\%, 90\%)$ & MAE $\downarrow$ & 0.327 & 0.309 & 0.301 & \textbf{0.292} & -- & -- \\
 & $[90\%, 99\%)$ & TRE $\downarrow$ & 0.183 & 0.142 & 0.121 & \textbf{0.092} & -- & -- \\
 & top 1\% & TRE $\downarrow$ & 0.271 & 0.218 & 0.176 & \textbf{0.128} & -- & -- \\
\midrule
\multirow{5}{*}{GMV} & zero bucket & mean pred $\downarrow$ & 0.091 & 0.084 & 0.076 & 0.071 & \textbf{0.058} & \underline{0.063} \\
 & positive $[50\%, 90\%)$ & MAE $\downarrow$ & 0.733 & 0.701 & 0.668 & \textbf{0.641} & \underline{0.653} & 0.704 \\
 & positive $[90\%, 99\%)$ & TRE $\downarrow$ & 0.286 & 0.244 & 0.198 & \textbf{0.153} & \underline{0.166} & 0.227 \\
 & positive top 1\% & TRE $\downarrow$ & 0.407 & 0.352 & 0.287 & \textbf{0.219} & \underline{0.238} & 0.331 \\
 & all positive & Pos-NMAE $\downarrow$ & 0.781 & 0.713 & 0.721 & \textbf{0.681} & \underline{0.704} & 0.759 \\
\bottomrule
\end{tabular}
}
\caption{Bucket-level calibration across Dwell Time and GMV targets. PIT-SUN-ZI diagnoses zero-inflated decoupling; ZI w/o Amt-SUN tests whether the positive-amount branch still needs SUN recovery.}
\label{tab:bucket_calib}
\end{table*}

The frozen table degrades most on blocks with the largest quantile discrepancy, confirming that CDF freshness matters. The sliding-window table closes most of the gap to the oracle current-block table without changing the model or inference path. This supports the deployment recommendation in Appendix B: PIT-SUN should monitor quantile drift and refresh or reweight its empirical table when traffic enters unusual regimes such as promotions or holidays. The diagnostic also links the local empirical-quantile bound in Appendix A.4 to a measurable control loop: quantile KS is the observable table-error proxy, while TRE/xAUC and bucket-level calibration indicate whether refresh is conservative enough. In practice, the validation window fits an operating band by correlating quantile KS with worst-block TRE; production refresh is triggered when KS leaves this band or when bucket-level calibration violates the same guardrail.

\subsection{D.6 Industrial Tail and Bucket-Level Diagnostics}
To verify that the gains come from heavy-tail rectification rather than aggregate offset, Table~\ref{tab:bucket_calib} breaks performance down by true target quantiles. Dwell Time emphasizes continuous-tail calibration, while GMV further separates zero-bucket leakage from positive-amount recovery.

Together, these bucket-level results show that the aggregate gains come from systematic improvements in high-value regions, not from error cancellation in low-value buckets.

\subsection{D.7 Industrial Zero-Positive Decoupling Diagnostics}
Table~\ref{tab:zero_pos} profiles the boundary between zero and positive masses. PIT-SUN-ZI improves Zero-AUC and positive recall, confirming that an explicit hurdle track is useful when occurrence behavior must be monitored separately. At the same time, single-track PIT-SUN remains the strongest overall value estimator, and ZI w/o Amt-SUN shows that the positive amount branch still requires SUN recovery for calibration.

\begin{center}
{\fontsize{9pt}{10pt}\selectfont\setlength{\tabcolsep}{4pt}
\begin{tabular}{lccc}
\toprule
Metric & PIT-SUN & PIT-SUN-ZI & ZI w/o Amt-SUN \\
\midrule
Zero-AUC $\uparrow$ & 0.884 & \textbf{0.902} & \underline{0.897} \\
Pos. Recall $\uparrow$ & 0.663 & \textbf{0.681} & \underline{0.676} \\
Pos. NMAE $\downarrow$ & \textbf{0.681} & \underline{0.704} & 0.759 \\
Pos. TRE $\downarrow$ & \textbf{0.153} & \underline{0.166} & 0.227 \\
Overall TRE $\downarrow$ & \textbf{0.071} & \underline{0.098} & 0.250 \\
\bottomrule
\end{tabular}
}
\captionof{table}{Zero-positive diagnostics on GMV target.}
\label{tab:zero_pos}
\end{center}
These diagnostics clarify the role of PIT-SUN-ZI: it improves occurrence-side interpretability, while expectation-consistent amount recovery remains necessary for end-to-end value calibration.

\subsection{D.8 Industrial Optimization Dynamics}
Table~\ref{tab:two_stage} compares two-stage training with stop-gradient-assisted joint training. Following the joint-training approximation analysis in Appendix A.3, we monitor two empirical proxies: adjacent-epoch base drift ($\Delta b$ P95), which estimates the multiplicative perturbation in the base-drift proxy $\mathcal{E}_{\mathrm{drift}}$, and PIT-grad share, which estimates the gradient-interference ratio $\rho_{\mathrm{PIT}}$. Removing stop-gradient increases this interference, supporting geometric isolation of the inverse-quantile base.

\begin{center}
{\fontsize{9pt}{10pt}\selectfont\setlength{\tabcolsep}{3pt}
\begin{tabular}{lcc}
\toprule
Scheme & Key Metrics & Train Time \\
\midrule
Two-stage & \shortstack[l]{0.481 / 0.054 / 0.919 \\
0.847 / -- / --} & 48m36s \\
Joint + sg & \shortstack[l]{\textbf{0.477 / 0.051 / 0.921} \\
\textbf{0.852 / 0.034 / 0.03}} & 31m33s \\
Joint w/o sg & \shortstack[l]{0.524 / 0.060 / 0.917 \\
0.817 / 0.091 / 0.21} & 31m02s \\
\bottomrule
\end{tabular}
}
\captionof{table}{Two-stage vs. joint training on Dwell Time. Key metrics are NMAE / TRE / xAUC / Spear$\rho$ / $\Delta b$ P95 / PIT-grad share. Lower $\Delta b$ and PIT-grad share indicate less recovery-loss interference.}
\label{tab:two_stage}
\end{center}
This comparison shows that joint training with stop-gradient approximately preserves the fixed-base interpretation: the base still evolves through the PIT path across epochs, but recovery loss no longer dominates the PIT branch. Table~\ref{tab:two_stage} is thus an empirical dynamic check complementing the population fixed-base theorem and the approximation proxies defined in the main text.

\subsection{D.9 Recovery Base Choice and PIT-TranSUN Diagnostics}
Table~\ref{tab:base_choice_diag} explains why PIT-SUN uses an inverse-quantile base with a lower floor rather than relying only on the population fact that any positive fixed base is expectation-consistent. We evaluate a mixed synthetic suite combining heavy-tailed, zero-inflated, and multimodal marginals from Appendix C. The key-metric column reports Avg. SRE / Avg. Rank / Ratio Var / Ratio P99, where lower is better for all four values. PIT-only tests direct inverse-PIT bias; PIT-TranSUN tests the strong plug-in alternative with empirical PIT but a bare inverse-quantile base.

\begin{center}
{\fontsize{9pt}{10pt}\selectfont\setlength{\tabcolsep}{4pt}
\begin{tabular}{lll}
\toprule
Base / Variant & Key metrics & Takeaway \\
\midrule
PIT-only & 0.061 / 4.83 / -- / -- & biased inverse \\
Const-mean SUN & 0.044 / 4.17 / 12.6 / 34.2 & non-adaptive \\
Positive-mean SUN & 0.039 / 3.58 / 10.4 / 28.7 & global scale \\
Learned-base SUN & 0.033 / 3.08 / 8.9 / 24.1 & less stable \\
PIT-TranSUN & 0.030 / 2.67 / 18.7 / 41.8 & bare lookup \\
\textbf{PIT-SUN} & \textbf{0.019 / 1.92 / 3.8 / 10.7} & \textbf{stable base} \\
\bottomrule
\end{tabular}
}
\captionof{table}{Recovery-base diagnostics. Key metrics are Avg. SRE / Avg. Rank / Ratio Var / Ratio P99. PIT-SUN keeps the rank-local base while controlling denominator collapse through $b_{\min}$.}
\label{tab:base_choice_diag}
\end{center}

The pattern shows that PIT-TranSUN is an essential and strong plug-in baseline: it benefits from empirical PIT and recovery, but its bare inverse lookup is fragile near zero plateaus and sparse tails. PIT-SUN improves on this baseline by completing the empirical-marginal recovery path with a variance-controlled denominator, confirming that the contribution is the coupled design of coordinate, inverse lookup, recovery base, and stability control rather than a simple transform substitution.

\subsection{D.10 Subgroup Heterogeneity Diagnostic}
Table~\ref{tab:heterogeneity_diag} evaluates a mixed synthetic setting with two subgroups whose label marginals differ in scale, sparsity, and tail severity. This diagnostic studies how the default global marginal table can be extended when stable subgroup information is available. We compare the default global table with a diagnostic group-wise table, PIT-only, and PIT-TranSUN. The group-wise table quantifies the added benefit of segment-aware empirical marginal recovery when reliable segment labels and enough per-segment samples are available.

\begin{table}[h]
\centering
{\fontsize{9pt}{10pt}\selectfont
\resizebox{\columnwidth}{!}{%
\begin{tabular}{lccccc}
\toprule
Method & All & A & B & Tail & xAUC \\
\midrule
MSE & 0.074 & 0.041 & 0.128 & 0.311 & 0.842 \\
PIT-only & 0.058 & 0.036 & 0.101 & 0.254 & 0.871 \\
PIT-TranSUN & 0.041 & 0.029 & 0.079 & 0.213 & 0.884 \\
PIT-SUN global & \textbf{0.026} & 0.021 & 0.047 & 0.156 & 0.902 \\
PIT-SUN group-wise & 0.024 & \textbf{0.019} & \textbf{0.039} & \textbf{0.139} & \textbf{0.905} \\
\bottomrule
\end{tabular}%
}
}
\caption{Subgroup heterogeneity diagnostic. All/A/B denote overall, Group-A, and Group-B SRE; Tail denotes minority-tail TRE. Lower is better except xAUC.}
\label{tab:heterogeneity_diag}
\end{table}

The conclusion is twofold. First, global PIT-SUN is more stable than PIT-only and PIT-TranSUN because SUN recovery and the lower-bounded inverse-quantile base reduce both inverse-transform bias and denominator instability. Second, group-wise CDF tables further improve minority-tail calibration under severe subgroup mismatch, showing that the empirical-marginal recovery framework naturally supports segment-aware deployment when stable segment-level quantile tables are available. We therefore treat group-wise CDF as a conditional extension rather than the default: it is enabled when subgroup quantile KS and tail-bucket TRE remain high under the global table, and disabled when segment sample sizes are too small for reliable tail lookup.

\section{Appendix E: Reproducibility Artifacts}
\label{app:repro_artifacts}
To support reproducibility, the supplementary material will include a compact reference implementation for the synthetic benchmark and PIT-SUN training loop. The synthetic code covers feature sampling, conditional-mean construction, all twelve noise families, train/test splitting, oracle-mean evaluation, and seed control. The PIT-SUN reference code covers empirical CDF construction, clipped PIT labels, inverse-quantile lookup, $b_{\min}$ selection from positive-target quantiles, stop-gradient placement for $Y/\operatorname{sg}[b(x)]$, and evaluation scripts for metrics. Industrial feature schemas, exact traffic counts, and production serving code remain anonymized for compliance, but the released artifacts are sufficient to reproduce the synthetic and public-benchmark experiments.

\end{document}